\documentclass[twoside]{article}

\usepackage[accepted]{aistats2022}
\usepackage{mathtools}
\usepackage[round]{natbib}
\usepackage[utf8]{inputenc} % allow utf-8 input
\usepackage[T1]{fontenc}    % use 8-bit T1 fonts
\usepackage{url}            % simple URL typesetting
\usepackage{booktabs}       % professional-quality tables
\usepackage{nicefrac}       % compact symbols for 1/2, etc.
\usepackage{microtype}      % microtypography
\usepackage{xcolor}         % colors
\usepackage{wrapfig}
\usepackage{subcaption}
\usepackage[ruled]{algorithm2e}
\usepackage{algorithmic}
\usepackage{amssymb,amsmath}
\usepackage{amsthm}
\usepackage{thmtools, thm-restate}
\usepackage[mathscr]{eucal} % for mathscr
\usepackage{natbib}
\usepackage{bbm}
\usepackage{arydshln}
\usepackage{enumitem}
\usepackage{cancel}
\usepackage{graphicx}
\usepackage{mdframed}

\definecolor{mydarkblue}{rgb}{0,0.08,0.45}
\usepackage[colorlinks=true,
    linkcolor=mydarkblue,
    citecolor=mydarkblue,
    filecolor=mydarkblue,
    urlcolor=mydarkblue]{hyperref} 
\usepackage[capitalize]{cleveref}

\declaretheorem{proposition}
\declaretheorem{lemma}

\crefname{assumption}{assumption}{Assumptions}
\crefname{figure}{Figure}{Figure}

\mdfdefinestyle{mdframedthmbox}{%
	leftmargin=.0\textwidth,
	rightmargin=.0\textwidth,%
	innertopmargin=0.75em,
	innerleftmargin=.5em,
	innerrightmargin=.5em,
}

\newenvironment{thmbox}
	{%
		\begin{mdframed}[style=mdframedthmbox]%
	}{% 
		\end{mdframed}%
	}
\newcommand{\myquote}[1]{\null~\\{\null\hspace{.05\textwidth}\begin{minipage}[t]{.90\textwidth} #1 \end{minipage}}}

\def\pih{{\pi_{t + \nicefrac{1}{2}}}}
\def\pitt{{\pi_{t+1}}}

\newcommand{\pit}{{\pi_{t}}}

\def\piopt{{\pi^\ast}}

\newcommand{\breg}[2]{D_{\phi}(#1, #2)}

\newcommand{\inner}[2]{\langle #1,\, #2 \rangle}
\newcommand{\normsq}[1]{|| #1 ||^{2}}
\newcommand{\norm}[1]{|| #1 ||}

\newcommand{\pidist}{{p^\pi}}
\newcommand{\pitdist}{{p^{\pi_{t}}}}

\newcommand{\pipdist}{{p^{\pi'}}}

\def\1{\bm{1}}

\DeclareMathAlphabet{\mathsfit}{\encodingdefault}{\sfdefault}{m}{sl}
\SetMathAlphabet{\mathsfit}{bold}{\encodingdefault}{\sfdefault}{bx}{n}

\def\gU{{\mathcal{U}}}

\def\sE{{\mathbb{E}}}

\def\sI{{\mathbb{I}}}

\def\sP{{\mathbb{P}}}

\def\sR{{\mathbb{R}}}

\newcommand{\E}{\mathbb{E}}

\newcommand{\R}{\mathbb{R}}

\DeclareMathOperator*{\argmax}{arg\,max}

\def\FMAPG/{{FMA-PG}}
\def\DFMAPG/{{DFMA-PG}}
\def\SFMAPG/{{SFMA-PG}}

\begin{document}
\runningtitle{A general class of surrogate functions for stable and efficient reinforcement learning}

\runningauthor{Vaswani, Bachem, Totaro, Müller, Garg, Geist, Machado, Castro, Le Roux}

\twocolumn[

\aistatstitle{A general class of surrogate functions for stable \\ and efficient reinforcement learning}

\aistatsauthor{Sharan Vaswani$^{1}$ \And  Olivier Bachem $^2$ \And Simone Totaro$^3$ \AND Robert Müller$^4$ \And Shivam Garg$^5$ \And  Matthieu Geist$^2$ \AND  Marlos C. Machado$^{5,6}$ \And Pablo Samuel Castro$^2$ \And Nicolas Le Roux$^{3,7,8}$}  

\vspace{5ex}

\aistatsaddress{$^1$Simon Fraser University \And $^2$Google Brain \And  $^3$Mila, Universit\'e de Montr\'eal \And  $^4$TU Munich \AND $^5$Amii, University of Alberta \And $^6$DeepMind \And  $^7$Microsoft Research\And  $^8$Mila, McGill } ]

\begin{abstract}
Common policy gradient methods rely on the maximization of a sequence of surrogate functions. In recent years, many such surrogate functions have been proposed, most without strong theoretical guarantees, leading to algorithms such as TRPO, PPO or MPO. Rather than design yet another surrogate function, we instead propose a general framework (FMA-PG) based on functional mirror ascent that gives rise to an entire family of surrogate functions. We construct surrogate functions that enable policy improvement guarantees, a property not shared by most existing surrogate functions. Crucially, these guarantees hold regardless of the choice of policy parameterization. Moreover, a particular instantiation of FMA-PG recovers important implementation heuristics (e.g., using forward vs reverse KL divergence) resulting in a variant of TRPO with additional desirable properties. Via experiments on simple reinforcement learning problems, we evaluate the algorithms instantiated by FMA-PG. The proposed framework also suggests an improved variant of PPO, whose robustness and efficiency we empirically demonstrate on the MuJoCo suite. 

% % is allowing us to leverage the associated theory to 
% Importantly, because the functional representation can be chosen independently of the parameterization, 
% these guarantees exist
% The functional perspective distinguishes between a policy's functional representation  and its parameterization . 
% With general function approximation 
% that offer strong theoretical guarantees, in particular that a partial minimization of that surrogate leads to an improvement of the policy.
% When designing policy gradient methods, it is common to define surrogate functions that are being optimized without requiring resampling. Howev
% , and helps design good surrogate functions that enable computationally efficient off-policy updates. 
% naturally results in computationally efficient off-policy updates. 
% For simple policy parameterizations, the FMA-PG framework ensures that the optimal policy is a fixed point of the updates. 
% It also allows us to handle complex policy parameterizations (e.g., neural networks) while guaranteeing policy improvement. 
% Our framework unifies several PG methods and opens the way for designing sample-efficient variants of existing methods. 
\end{abstract}
\section{INTRODUCTION}
\label{sec:introduction}
Policy gradient (PG) methods~\citep{williams1992simple, sutton2000policy,konda2000actor, kakade2002natural} are an important class of model-free methods in reinforcement learning. They enable a differentiable policy parameterization and can easily handle function approximation and structured state-action spaces. PG methods based on REINFORCE~\citep{williams1991function} are equipped with strong theoretical guarantees in restricted settings~\citep{agarwal2019optimality,mei2020global, cen2020fast}. For these methods, each policy update requires recomputing the policy gradient. This in turn requires interacting with the environment or the simulator which can be computationally expensive. 

On the other hand, methods such as TRPO~\citep{schulman2015trust}, PPO~\citep{schulman2017proximal} and MPO~\citep{abdolmaleki2018maximum} support off-policy updates, i.e. they can update the policy without requiring additional interactions  with the environment. These  methods are efficiently implementable and have good empirical performance~\citep{openaibaselines}. All of these methods rely on constructing \emph{surrogate functions} of the policy, then updating the policy by maximizing these surrogates. Unfortunately, most of these surrogate functions (including those for PPO, TRPO and MPO) do not have strong theoretical guarantees. Consequently, this class of PG methods only has performance guarantees in the tabular setting~\citep{kakade2002approximately, schulman2015trust,neuJG17,geist2019theory, shani2020adaptive}, and some of these can even fail to converge in simple scenarios~\citep{hsu2020revisiting}. More importantly, \emph{there is no systematic way to design theoretically principled surrogate functions, or a unified framework to analyze their properties}. We address these issues through the following contributions.

\textbf{Functional mirror ascent for policy gradient}: In~\cref{sec:fmapg}, we construct surrogate functions using mirror ascent on a functional representation of the policy itself,  rather than on its parameters. We call this approach functional mirror ascent (FMA) and derive its update for policy gradient methods. The FMA update results in a surrogate function that is independent of the policy parameterization. We use it to propose \FMAPG/ (FMA for PG), a general framework for constructing surrogate functions and introduce a generic policy optimization algorithm that relies on approximately maximizing a sequence of surrogate functions.

\textbf{Theoretical guarantees for \FMAPG/}: In~\cref{sec:guarantees}, we explain the theoretical advantages of using \FMAPG/. In particular, we describe a sufficient condition that guarantees that maximizing the sequence of surrogate functions instantiated by \FMAPG/ will result in monotonic policy improvement and ensure convergence to a stationary point. Crucially, \emph{these guarantees hold regardless of the choice of policy parameterization}. 

\textbf{Instantiating the \FMAPG/ framework}: In~\cref{sec:instantiation}, we instantiate the \FMAPG/ framework with two common functional representations -- direct and softmax representations. For each of these, we compare the resulting surrogate function to existing methods in the literature. For each representation, we prove that a specific surrogate function instantiated by \FMAPG/ satisfies the sufficient condition in~\cref{sec:guarantees}. Consequently, maximizing it guarantees monotonic policy improvement for \emph{arbitrarily complicated policy parameterizations} including neural networks. Such a property is not shared by existing surrogate functions including those for PPO, TRPO and MPO.

For the softmax functional representation, \FMAPG/ results in a surrogate function that is a more stable variant of TRPO~\citep{schulman2015trust} and MDPO~\citep{tomar2020mirror}. Moreover, it recovers implementation heuristics (e.g. using forward vs reverse KL divergence) in a principled manner. Additionally, in~\cref{app:svg}, we show that \FMAPG/ can handle stochastic value gradients~\citep{heess2015learning}. 

\textbf{Experimental evaluation}: Finally, in~\cref{sec:experiments}, we evaluate the performance of surrogate functions instantiated by \FMAPG/ on simple bandit and reinforcement learning settings. \FMAPG/ also suggests a variant of PPO~\citep{schulman2017proximal}, whose robustness and efficiency we demonstrate on continuous control tasks in the MuJoco environment~\citep{todorov2012mujoco}. 
% -------------- DUMP --------
% Consequently, there are numerous discrepancies between the theory and practice of these methods~\citep{lazic2021optimization,mei2019principled,engstrom2019implementation}. 
% Our framework unifies different perspectives and provides a principled way to develop and analyze computationally efficient PG methods. 

% In particular, we show that FMA can be interpreted as the repeated application of a policy improvement and a projection (onto the set of feasible policies) operator~\citep{ghosh2020operator}. For simple policy parameterizations, we prove that the FMA updates are consistent~\citep{ghosh2020operator} ensuring that the optimal policy (in the class) is a fixed point of the resulting PG method.   

%  

% In the tabular setting with finite states and actions, we show that the resulting FMA updates recover conservative policy iteration~\citep{kakade2002approximately} and REINFORCE-based methods analyzed in~\citep{agarwal2019optimality,mei2020global,cen2020fast}.

\section{PROBLEM FORMULATION}
\label{sec:problem-formulation}
\vspace{-2ex}
We consider an infinite-horizon discounted Markov decision process (MDP)~\citep{Puterman1994} defined by the tuple $\mathcal{M} = \langle \mathscr{S}, \mathscr{A}, p, r, d_0, \gamma \rangle$ where $\mathscr{S}$ and $\mathscr{A}$ is the set of states and actions respectively, $p : \mathscr{S} \times \mathscr{A} \rightarrow \Delta^\mathscr{S}$ the transition probability function, $r : \mathscr{S} \times \mathscr{A} \rightarrow \sR$ the reward function, $d_0$ the initial distribution over states, and $\gamma \in [0, 1)$ the discount factor. Each policy $\pi$ induces a distribution $\pidist(\cdot | s)$ over actions for each state $s$. It also induces a measure $d^\pi$ over states such that $d^\pi(s) = \sum_{\tau = 0}^\infty \gamma^\tau \sP(s_{\tau} = s \mid s_0 \sim d_0, a_{\tau} \sim \pidist(a_\tau|s_{\tau}))$. Similarly, we define $\mu^\pi$ as the induced measure over state-action pairs induced by policy $\pi$, implying that $\mu^\pi(s,a) = d^\pi(s) \pidist(a|s)$ and $d^\pi(s) = \sum_a \mu^\pi(s, a)$. The expected discounted return for $\pi$ is defined as $J(\pi) = \mathbb{E}_{s_0, a_0, \ldots} [\sum_{\tau=0}^\infty \gamma^\tau r(s_\tau, a_\tau)]$, where $s_0 \sim d_0, a_\tau \sim \pidist( a_\tau| s_\tau),$ and $s_{\tau+1} \sim p(s_{\tau+1} | s_\tau, a_\tau)$. Given a set of feasible policies $\Pi$, the objective is to compute the policy that maximizes $J(\pi)$. We define $\piopt := \argmax_{\pi \in \Pi} J(\pi)$ as the \emph{optimal policy}.

We call the set of distributions $\pidist(\cdot | s)$ for each $s$ or the measure $d^\pi$ \emph{functional representations} of the policy $\pi$. Note that a single policy policy $\pi$ can have multiple functional representations. In general, optimizing $J$ directly with respct to any functional representation of $\pi$ is intractable. Consequently, the standard approach is to parameterize $\pi$ by a set of parameters $\theta \in \R^d$ and to directly optimize $J$ with respect to $\theta$. However, it is critical to remember that \emph{the functional representation of a policy is independent of its parameterization}.

% As this approach can be expensive, mainly because it involves resampling actions and states after each update, the most popular algorithms use an inner-loop / outer-loop mechanism where they optimize a sequence of surrogate functions, only sampling between two steps of the outer loop.

There are other possible functional representations of a policy besides the two mentioned above. For example, since $\pidist(\cdot|s)$ is a probability distribution, one can write $\pidist(a|s) =  \nicefrac{\exp(z^\pi(a,s))}{\sum_{a^\prime} \exp(z^\pi(a^\prime,s))}$, and represent $\pi$ as the set of $z^\pi(a,s)$ for each $(a,s)$ pair. We call this particular functional representation the \emph{softmax representation}, as opposed to the set of $p^\pi(a|s)$ which we call the \emph{direct representation}. In the next section, we describe how to use the functional representation of a policy to derive a surrogate function. Although multiple functional representations can be equivalent in the class of policies they define, they result in different surrogate functions (\cref{sec:direct,sec:softmax}). Finally, we note that functional representations are not limited to stochastic policies and one can, for instance, represent a deterministic, stationary policy by specifying the state-action mapping for each state (\cref{app:svg}). 

\section{FUNCTIONAL MIRROR ASCENT FOR POLICY GRADIENT}
\label{sec:fmapg}
% In this section, we explicit the difference between a policy's functional representation and its parameterization, then state the FMA update for PG methods and fully specify the \FMAPG/ framework. 

% A policy's \emph{functional representation} (denoted as $\pi$) defines its sufficient statistics and can be non-parametric. For example, we may define a policy via a distribution $\pidist(\cdot|s)$ over the actions for each state $s \in \mathscr{S}$, which we call the \emph{direct representation}. Such a representation is used for stochastic policies typically used with policy gradient algorithms~\citep{sutton18book}. Since $\pidist(\cdot|s)$ is a probability distribution, an equivalent form is the \emph{softmax representation} $\pidist(a|s) =  \nicefrac{\exp(z^\pi(a,s))}{\sum_{a^\prime} \exp(z^\pi(a^\prime,s))}$ where the policy is specified by the $z^\pi(a,s)$ variables. Note that though the direct and softmax representations are equivalent in the class of policies they define, they will result in different surrogate functions (\cref{sec:direct,sec:softmax}). Alternatively, one can specify a stochastic policy by its state-occupancy measures~\citep{Puterman1994} or represent a deterministic, stationary policy by specifying the state-action mapping for each state. The functional representation affects the form of the surrogate function, but is never made explicit.

In the previous section, we defined the functional representation of a policy. However, as we mentioned, typically, one cannot optimize $J$ with respect to these representations directly, in which case the policy $\pi$ is  parameterized. While the functional representation defines a policy's sufficient statistics, the \emph{policy parameterization} specifies the practical realization of these statistics and defines the set $\Pi$ of realizable (representable) policies. The parameterization is independent of the functional representation, is explicit and determined by a \emph{model} with parameters $\theta$. For example, we could represent a policy by its state-action occupancy measure and use a linear parameterization to realize this measure, implying $\mu^\pi(s,a | \theta) = \inner{\theta}{\phi(s,a)}$, where $\theta$ is the parameter to be optimized and $\phi(s,a)$ are the known features providing information about the state-occupancy measures. Similarly, we could use a neural-network parameterization for the variables that define a policy in its softmax representation, rewriting $z^\pi(a,s) = z^\pi(a,s | \theta)$. In order to compare to existing methods~\citep{agarwal2019optimality, mei2020global}, we also define a \emph{tabular parameterization}. For a finite state-action MDP with $S$ states and $A$ actions, choosing a tabular parameterization with the softmax representation results in $\theta \in \R^{SA}$ such that $\forall s \in \mathscr{S}, a \in \mathscr{A}$, $z^\pi(a,s |\theta) = \theta_{s,a}$. 

Next, we describe a form of mirror ascent to directly update a policy's functional representation.

\subsection{Functional Mirror Ascent Update}
\label{sec:fma-update}
To state the functional mirror ascent (FMA) update, we define a strictly convex, differentiable function $\phi$ as the mirror map. We denote by $\breg{\pi}{\mu}$ the Bregman divergence associated with the mirror map $\phi$ between policies $\pi$ and $\mu$. Each iteration $t \in [T]$ of FMA consists of the update and projection steps~\citep{bubeck2015convex}: \cref{eq:fmd-op-update} computes the gradient $\nabla_{\pi} J(\pi_t)$ \emph{with respect to the policy's functional representation} and updates $\pit$ to $\pih$ using a step-size $\eta$; \cref{eq:fmd-op-proj} computes the Bregman projection of $\pih$ onto the class of realizable policies, obtaining $\pitt$.    
\begin{align}
    \pih = (\nabla\phi)^{-1}\left(\nabla\phi(\pi_t) + \eta \nabla J(\pi_t)\right), \label{eq:fmd-op-update} \\
    \pitt = \arg\min_{\pi \in \Pi} D_\phi(\pi, \pi_{t+1/2}). \label{eq:fmd-op-proj}
\end{align}
The above FMA updates can also be written as \citep[c.f.][]{bubeck2015convex}: 
\begin{align}
\pitt &= \arg\max_{\pi \in \Pi} \left[ \inner{\pi}{\nabla_\pi J(\pi_t)} - \frac{1}{\eta}D_\phi(\pi, \pi_t) \right].     
\label{eq:fmd-update}
\end{align}
Note that the FMA update is solely in the functional space, and is specified by the choice of the functional representation and mirror map. The update requires solving a sub-problem to project the updated policy onto the set $\Pi$. Since the policy parameterization defines the set $\Pi$ of realizable policies, it influences the difficulty of solving this projection sub-problem as well as the final policy $\pi_{t+1}$. For simple policy parameterizations such as tabular or when using a linear model, the set $\Pi$ is convex and the minimization in~\cref{eq:fmd-update} can be done exactly. When using more complex policy parameterizations (e.g. deep neural network), the set of realizable policies $\Pi$ can become arbitrarily complicated and non-convex, making the projection in~\cref{eq:fmd-update} infeasible. The \FMAPG/ framework overcomes this issue as follows.

% Instead, to handle arbitrary policy parameterizations, we reparameterize the constrained optimization in~\cref{eq:fmd-update} as an unconstrained optimization problem, and completely specify the \FMAPG/ framework.  

\subsection{\FMAPG/ Framework}
\label{sec:fmapg-framework}
We assume that $\Pi$ consists of policies that are realizable by a model parameterized by $\theta \in \R^d$. Throughout the paper, we will use $\pi$ to refer to a policy's functional representation, whereas $\pi(\theta)$ will refer to the parametric realization of $\pi$. We do not impose any restriction on the parameterization and any generic model (e.g. neural network) can be used to parameterize $\pi$. The choice of the policy parameterization is implicit in the $\pi(\theta)$ notation. For the special case of the tabular parameterization, $\pi = \pi(\theta) = \theta$.

Solving~\cref{eq:fmd-op-proj} iteratively may be interpreted as finding a path that starts from $\pi_{t+1/2}$ and gradually gets closer to the set $\Pi$. In this view, an approximate solution would be a point along that path that is not in the set $\Pi$, and consequently not realizable by a vector $\theta$. Another perspective is to interpret solving \cref{eq:fmd-op-proj} as finding a path \emph{within $\Pi$} that starts from $\pi_t$, the previous policy (already in $\Pi$), and gets closer to $\pi_{t+1/2}$ (potentially outside $\Pi$). Any point along such a path is within $\Pi$ and is thus realizable. In other words, we replace \eqref{eq:fmd-op-proj} with another problem with the same solution:
\begin{align}
\arg\min_{\pi \in \Pi} D_\phi(\pi, \pi_{t+1/2}) &= \arg\min_{\theta \in \R^d} D_\phi(\pi(\theta), \pi_{t+1/2}) \; .\label{eq:equivalence}
%     \max_{\pi \in \Pi}J(\pi) &= \max_{\theta \in \R^d} J(\pi(\theta)) 
% \end{align}
% \begin{align}
\end{align}
With this reparameterization, no projection is required and the update in~\cref{eq:fmd-update} can be written as a parametric, unconstrained optimization problem. This is a critical property as it makes \FMAPG/ applicable to any policy parameterization. 

In particular, if $\pit = \pi(\theta_t)$, $\theta_{t+1} \in \R^d$ is the solution to the RHS of~\cref{eq:equivalence} and $\pitt = \pi(\theta_{t+1})$, then~\cref{eq:fmd-update} can be written as the maximization of a \emph{surrogate function},  $\theta_{t+1} = \argmax_{\theta \in \R^d} \ell_t^{\pi, \phi, \eta}(\theta)$, where  
\begin{align}
\ell_t^{\pi, \phi, \eta}(\theta) & := J(\pi(\theta_t)) + \inner{\pi(\theta) - \pi(\theta_t)}{\nabla_\pi J(\pi(\theta_t))} \nonumber \\ 
& - \frac{1}{\eta}D_\phi(\pi(\theta), \pi(\theta_t)) \,.  \label{eq:fmd-update-param}    %
\end{align}
% Further, the choice of parameterization might affect the difficulty of finding the optimal solution but not that of finding approximate solutions, although their quality might.
The surrogate function $\ell_t^{\pi, \phi, \eta}(\theta)$ is a function of $\theta$, but it is specified by the choice of the functional representation, the mirror map $\Phi$, and the step-size $\eta$. Note that as compared to~\cref{eq:fmd-update}, in~\cref{eq:fmd-update-param}, we added terms independent of $\theta$ which do not change the $\argmax$ but will prove useful to prove guarantees in~\cref{sec:guarantees}. We have thus used the FMA update in~\cref{eq:fmd-update} to specify a family of surrogate functions that can be used with any policy parameterization. We refer to this general framework of constructing surrogates for policy gradient methods as \FMAPG/.  

% \looseness=-1
The surrogate function in~\cref{eq:fmd-update-param} is non-concave in general and can be maximized using a gradient-based algorithm. We will use $m$ gradient steps with a step-size $\alpha$ to maximize $\ell_t^{\pi, \phi, \eta}(\theta)$. With this choice, we can now state a generic policy optimization algorithm (pseudo-code in~\cref{alg:generic}). We see that the surrogate function $\ell_t^{\pi, \phi, \eta}$ acts as a ``guide'' for the parametric updates in the inner loop, similar to the supervised learning method proposed by~\citet{johnson2020guided}.  
\begin{algorithm}[!h]
\caption{Generic policy optimization}
\label{alg:generic}
\textbf{Input}: $\pi$ (choice of functional representation), $\theta_0$ (initial policy parameterization), $T$ (PG iterations), $m$ (inner-loops), $\eta$ (step-size for functional update), $\alpha$ (step-size for parametric update) \\
\For{$t \leftarrow 0$  \KwTo  $T-1$}{
    Compute gradient $\nabla_\pi J(\pi_t)$ and form function $\ell_t^{\pi, \phi, \eta}(\theta)$ as in~\cref{eq:fmd-update-param} \\
    Initialize inner-loop: $\omega_0 = \theta_{t}$ \\
    \For{$k \leftarrow 0$   \KwTo  $m$}{
    $\omega_{k+1} = \omega_{k} + \alpha \nabla_{\omega} \ell_t^{\pi, \phi, \eta}(\omega_k)$
    }
    $\theta_{t+1} = \omega_{m}$ \\
    $\pitt = \pi(\theta_{t+1})$
}
Return $\theta_{T}$
\end{algorithm}

% functional mirror ascent perspective and the resulting \FMAPG/ framework help design theoretically principled surrogate functions. 

\vspace{-2ex}
\section{THEORETICAL GUARANTEES}
\label{sec:guarantees}
In this section, we explain the theoretical advantage of using surrogate functions instantiated by \FMAPG/. Recall that the policy is updated through the (potentially approximate) maximization of~\cref{eq:fmd-update-param}. To guarantee that maximizing the surrogate function improves the resulting policy, i.e. $J(\pitt) \geq J(\pit)$, a sufficient condition is to have $\ell_t(\theta) \leq J(\pi(\theta))$ \emph{for all} $\theta$. Indeed, if $\ell_t$ is a uniform lower-bound on $J$, then,
\begin{align*}
& J(\pi_{t+1}) = J(\pi(\theta_{t+1})) \geq \ell_t(\theta_{t+1}) \\
& \geq \ell_t(\theta_{t}) & \tag{By maximizing the surrogate function} \\
& = J(\pi(\theta_t)) = J(\pit) & \tag{From~\cref{eq:fmd-update-param}}
\end{align*}

For stating a more practical condition that guarantees that the surrogate function is a uniform lower-bound on $J$, we prove the following proposition in~\cref{app:proofs-fmapg}.
\vspace{-4ex}
\begin{restatable}[Guarantee on surrogate function]{proposition}{smooth}
\label{prop:eta_smoothness}
The surrogate function $\ell_t^{\pi, \phi, \eta}$ is a lower bound of $J$ if and only if $J + \frac{1}{\eta}\phi$ is a convex function of $\pi$.
\end{restatable}
The above proposition shows that the desired property is guaranteed by selecting an appropriate value of $\eta$ that only depends on properties of $J$ and the mirror map $\Phi$ in the functional space. Once again, we emphasize that the guarantees offered by the surrogate function are \emph{independent of the parameterization.}

We have seen that if the surrogate is a uniform lower bound on $J$, then the equality of the two functions at $\theta = \theta_t$ (from~\cref{eq:fmd-update-param}) guarantees that any improvement of the surrogate leads to an improvement of $J$. The following result states that improvement in the surrogate can be guaranteed provided that the parametric step-size $\alpha$ is chosen according to the smoothness of the surrogate function.
\begin{restatable}[Guaranteed policy improvement for~\cref{alg:generic}]{theorem}{improvementparametric}
\label{thm:improvement-parametric}
Assume that $\ell_t$ is $\beta$-smooth w.r.t. the Euclidean norm and that $\eta$ satisfies the condition of Proposition~\ref{prop:eta_smoothness}. Then, for any $\alpha \leq \nicefrac{1}{\beta}$, iteration $t$ of~\cref{alg:generic} guarantees $J(\pitt) \geq J(\pi_t)$ for any number $m$ of inner-loop updates. 
\end{restatable}
Note that~\cref{alg:generic} and the corresponding theorem can be easily extended to handle  stochastic parametric updates. This will guarantee that $\E [J(\pitt)] \geq J(\pi_t)$ where the expectation is over the sampling in the parametric SGD steps. Similarly, both the algorithm and theoretical guarantee can be generalized to incorporate the relative smoothness of $\ell_t(\theta)$ w.r.t. a general Bregman divergence~\citep{lu2018relatively}. 

For rewards in $[0,1]$, $J(\pi)$ is upper-bounded by $\frac{1}{1 - \gamma}$, and hence monotonic improvements to the policy guarantee convergence to a stationary point. We emphasize that the above result holds for \emph{any arbitrarily complicated policy parameterization}. Hence, a successful PG method (one that reliably improves the policy) relies on appropriately setting two step-sizes: $\eta$ at the functional level and $\alpha$ at the parametric level.

\section{INSTANTIATING \FMAPG/}
\label{sec:instantiation}
% \looseness=-1
We now instantiate the \FMAPG/ framework with two common functional representations: the direct representation (\cref{sec:direct}) and the softmax representation (\cref{sec:softmax}), deriving values for $\eta$ for each.

\subsection{Direct Functional Representation}
\label{sec:direct}
In the direct functional representation, the policy $\pi$ is represented by the set of distributions $\pidist(\cdot|s)$ over actions for each state $s \in \mathscr{S}$. Using the policy gradient theorem~\citep{sutton18book}, in this case, $\frac{\partial J(\pi)}{\partial \pidist(a|s)} = d^\pi(s) Q^\pi(s,a)$. Since $\pidist(\cdot | s)$ is a set of distributions (one for each state), we define the mirror map as $\phi(\pi) = \sum_{s \in \mathscr{S}} w(s) \, \phi(\pidist(\cdot | s))$, where $w(s)$ is a positive weighting on the states $s$. Note that the positive weights ensure that $\phi$ is a valid mirror-map. The resulting Bregman divergence is $D_\phi(\pi, \pi') = \sum_s w(s) D_\phi(\pidist(\cdot | s), \pipdist(\cdot | s))$, that is, the weighted sum of the Bregman divergences between the action distributions in state $s$. By choosing $w(s)$ equal to $d^{\pi_t}(s)$, and parameterizing the functional representation, i.e. $\pitdist(\cdot | s) = \pidist(\cdot | s, \theta_t)$, we obtain the following form of the surrogate function:
\begin{align}
\ell_t^{\pi, \phi, \eta}(\theta) &= \E_{(s,a) \sim \mu^{\pit}} \left[
\left(Q^{\pi_t}(s, a) \, \frac{\pidist(a | s, \theta)}{\pidist(a | s, \theta_t)} \right) \right] \nonumber \\ 
& - \frac{1}{\eta} \E_{s \sim d^{\pit}} \left[ D_\phi(\pidist(\cdot | s, \theta), \pidist(\cdot | s, \theta_t)) \right] ,
\label{eq:fmd-direct-statewise-practical}
\end{align}
where the constants independent of $\theta$ were omitted. By choosing $\phi$ and $\eta$, the above surrogate function can be used with~\cref{alg:generic}. We now discuss how to set $\eta$ that guarantees monotonic policy improvement when using the above surrogate function with the negative entropy mirror map, i.e. $\phi_{NE}(\pidist(\cdot | s)) = -\sum_a \pidist(a | s) \log \pidist(a | s)$.  
\begin{restatable}[Improvement guarantees for direct functional representation]{proposition}{improvdirect}
\label{prop:eta_direct}
Assuming that the rewards are in $[0,1]$, when using the surrogate function in~\cref{eq:fmd-direct-statewise-practical} with the mirror map chosen to be the negative entropy, then $J \geq \ell_t^{\pi, \phi, \eta}$ for $\eta \leq \frac{(1 - \gamma)^3}{2\gamma |A|}$.
\end{restatable}
This proposition is proved in~\cref{app:proofs-instantiation}. Using the argument in~\cref{sec:guarantees}, we can infer that using the direct functional representation with the negative entropy mirror map and $\eta \leq \frac{(1 - \gamma)^3}{2\gamma |A|}$ ensures monotonic policy improvement for \emph{any policy parameterization}. 

Next, we discuss how the surrogate function in~\cref{eq:fmd-direct-statewise-practical} and the resulting algorithm is related to existing methods. When using a tabular parameterization, i.e. when $\pi(\theta) = \theta$, we make the following connections:

\textbf{Connection to uniform TRPO and MDPI}: With the tabular parameterization, the proposed update is similar to the update in uniform TRPO~\citep{shani2020adaptive} and Mirror Descent Modified Policy Iteration~\citep{geist2019theory}. 

\textbf{Connection to CPI}: For finite states and actions, when using a tabular parameterization, the first term in~\cref{eq:fmd-direct-statewise-practical} becomes the same as in conservative policy iteration (CPI)~\citep{kakade2002approximately}. In CPI, the authors first derive the form $\sum_s d^{\pi}(s) \sum_a \pidist(a|s) Q^{\pi_t}(s,a)$, then use a mixture policy to ensure that $\pi$ is ``close'' to $\pi_t$ and justify replacing $d^{\pi}(s)$ in the above expression by $d^{\pi_t}$. On the other hand, we use the \FMAPG/ framework to directly derive~\cref{eq:fmd-direct-statewise-practical} and allow for the use of any Bregman divergence to ensure the proximity between $\pi$ and $\pi_t$. While we derive the CPI update from an unconstrained optimization viewpoint, CPI has also been connected to constrained optimization with an equivalence to functional Frank-Wolfe~\citep{scherrer2014local}.

\textbf{Connection to REINFORCE-based methods}: For finite states and actions, when using a tabular parameterization and \cref{alg:generic} with $m = \infty$ (exact minimization of the surrogate), if we choose the (i) squared Euclidean distance as the mirror map, the proposed update is the same as standard REINFORCE~\citep{williams1991function,agarwal2019optimality} and (ii) negative entropy as the mirror map (implying that the resulting Bregman divergence is the KL divergence), the proposed update is equal to natural policy gradient ~\citep{kakade2001natural}. 

\textbf{Comparison to MDPO}: With a direct functional representation, negative entropy mirror map and a general policy parameterization, the resulting \FMAPG/ update is similar to MDPO~\citep{tomar2020mirror}. The only difference between the two updates is that MDPO involves the advantage $A^{\pit}$ instead of the $Q^{\pit}$ term in~\cref{eq:fmd-direct-statewise-practical}. Since both $A^{\pit}$ and $Q^{\pit}$ are independent of $p^\pi$, this difference does not matter for gradient-based algorithms maximizing the surrogate (see caption of~\cref{table:ablation_study} in  \cref{sec:tabular_algorithmic_details}). Hence, MDPO directly falls under the \FMAPG/ framework. 

The above formulation has two main shortcomings. First, it involves $\pidist(a | s, \theta)$, which means that for each parametric update, either (i) the actions need to be resampled on-policy, or (ii) the update involves an importance-sampling ratio $\nicefrac{\pidist(a | s, \theta)}{\pidist(a | s, \theta_t)}$ like in~\cref{eq:fmd-direct-statewise-practical}. This requires clipping the ratio for  stability, and can potentially result in overly conservative updates~\citep{schulman2017proximal}. Moreover, with the mirror map as the negative entropy, the Bregman divergence is the \emph{reverse} KL divergence, i.e. $D_\phi(\pidist(\cdot | s, \theta), \pidist(\cdot | s, \theta_t)) = \text{KL}(\pidist(\cdot | s, \theta) || \pidist(\cdot | s, \theta_t))$. The reverse KL divergence makes this objective \emph{mode seeking}, in that the policy $\pi$ might only capture a subset of the actions covered by $\pi_t$. Past works have addressed this issue either by adding entropy regularization~\citep{geist2019theory, shani2020adaptive}, or by simply reversing the KL, using the \emph{forward KL}: $\text{KL}(\pidist(\cdot | s, \theta_t) || \pidist(\cdot | s, \theta))$~\citep{mei2019principled}. However, using entropy regularization results in a biased policy, whereas the forward KL does not correspond to a valid Bregman divergence in $p^\pi$ and can converge to a sub-optimal policy. We now show how \FMAPG/ with the softmax representation addresses both these issues in a principled way, providing a theoretical justification to heuristics that are used to improve PG methods.

\subsection{Softmax Functional Representation}
\label{sec:softmax}
Since $\pidist(\cdot | s)$ is a distribution, it has an equivalent softmax representation that we study in this section. The softmax functional representation results in the FMA update on the logits $z^\pi(a, s)$ of the conditional distributions $\pidist(a|s)$. Formally, $\pidist(a | s)  = \frac{\exp(z^\pi(a, s))}{\sum_{a'} \exp(z^\pi(a', s))}$ and the policy gradient theorem yields $\frac{\partial J(\pi)}{\partial z^\pi(a, s)} = d^{\pi}(s) A^\pi(s, a) \pidist(a | s)$. Here, $A^\pi(s,a)$ is the advantage function equal to $Q^\pi(s,a) - V^\pi(s)$. Similar to~\cref{sec:direct}, we use a mirror map $\phi_{z}(z)$ that decomposes across states, i.e. $\phi_z(z) = \sum_s w(s) \, \phi_z(z^\pi(\cdot, s))$ for some positive weighting $w$. We denote the corresponding Bregman divergence as $D_{\phi_z}$ and choose $w(s) = d^{\pi_t}(s)$.
Parameterizing the logits as $z^\pi(a, s, \theta)$ and noting that $\pitdist(a|s) = \pidist(a|s, \theta_t)$, we obtain the following form of the surrogate function:
\begin{align}
& \ell_t^{\pi, \phi, \eta}(\theta) = E_{(s, a) \sim \mu^{\pi_t}} \left[ A^{\pi_t}(s, a) \, z^\pi(a, s, \theta) \right] \nonumber \\ 
& \qquad - \frac{1}{\eta}\sum_s w(s) \, D_{\phi_z}\left(z^\pi(\cdot, s, \theta), z^\pi(\cdot, s, \theta_t) \right) \;. \label{eq:fmd-softmax-statewise}
\end{align}

We now discuss how the surrogate function in~\cref{eq:fmd-softmax-statewise} and the resulting algorithm relate to existing methods.

\textbf{Connection to REINFORCE-based methods}: For finite states and actions and when using a tabular parameterization and the squared Euclidean mirror map, \cref{alg:generic} with $m = 1$ leads to the same update as that of policy gradient with the softmax parameterization~\citep{agarwal2019optimality, mei2020global}. 

A more interesting surrogate emerges when $\phi$ is the \emph{logsumexp}, i.e. $\phi_z(z) = \sum_s w(s) \log \left(\sum_a \exp(z^\pi(a, s))\right)$, and $w(s) = d^{\pi_t}(s)$. Then, 
\begin{align}
\ell_t^{\pi, \phi, \eta}(\theta) = E_{(s, a) \sim \mu^{\pi_t}} \left[\left(A^{\pi_t}(s, a) + \frac{1}{\eta}\right)\log \frac{\pidist(a | s, \theta)}{\pidist(a | s, \theta_t)}\right], 
\label{eq:fmd-softmax-kl-practical}
\end{align}
omitting the constant terms independent of $\theta$. The full derivation of this computation can be found in~\cref{prop:fmd-softmax-kl} of \cref{app:proofs-instantiation}. We now discuss how to set $\eta$ that guarantees monotonic policy improvement when using the above surrogate function.
\begin{restatable}[Improvement guarantees for softmax functional representation]{proposition}{improvsoftmax}
\label{prop:eta_softmax}
Assuming that the rewards are in $[0,1]$, then the surrogate function in~\cref{eq:fmd-softmax-kl-practical} satisfies $J \geq \ell_t^{\pi, \phi, \eta}$ for $\eta \leq 1 - \gamma$.
\end{restatable}
This proposition is proved in~\cref{app:proofs-instantiation}. As before, we can infer that using the softmax functional representation with the logsumexp mirror map and $\eta \leq 1 - \gamma$ ensures monotonic policy improvement for \emph{any policy parameterization}. Although we have used the same $\eta$ for all states $s$, the updates in~\cref{eq:fmd-direct-statewise-practical,eq:fmd-softmax-kl-practical} can accommodate a different step-size $\eta(s)$ for each state. This is likely to yield tighter lower bounds and larger improvements in the inner loop. Determining such step-sizes is left for future work.

Unlike the formulation in~\cref{eq:fmd-direct-statewise-practical}, we see that~\cref{eq:fmd-softmax-kl-practical} relies on the logarithm of the importance sampling ratios. Moreover,~\cref{eq:fmd-softmax-kl-practical} can be written as
\begin{align}
\ell_t^{\pi, \phi, \eta} &= \E_{s \sim d^{\pit}} \bigg[ \E_{a \sim \pitdist} \left(A^{\pi_t}(s, a) \log \frac{\pidist(a | s, \theta)}{\pidist(a | s, \theta_t)}  \right) \nonumber \\ 
& - \frac{1}{\eta} \text{KL} (\pidist(\cdot | s, \theta_t) || \pidist(\cdot | s, \theta)) \bigg].
\label{eq:fmd-softmax-kl-practical-2}
\end{align}
Comparing to~\cref{eq:fmd-direct-statewise-practical}, we observe that the KL divergence is in the \emph{forward} direction and is \emph{mode covering}. This naturally prevents a mode-collapse of the policy and encourages exploration. We thus see that \FMAPG/ is able to recover an implementation heuristic (forward vs reverse KL) in a principled manner. Moreover, we can interpret~\cref{eq:fmd-softmax-kl-practical-2} as a variant of TRPO with desirable properties, as we discuss next. 

% \looseness=-1
\textbf{Comparison to TRPO}: Comparing~\cref{eq:fmd-softmax-kl-practical-2} to the TRPO update~\citep{schulman2015trust}, $\argmax_{\theta \in \R^d} \E_{(s, a) \sim \mu^{\pi_t}} [A^{\pi_t}(s, a) \, \frac{\pidist(a|s, \theta)}{\pidist(a|s, \theta_t)} ]$, such that $\E_{s \sim d^{\pi_t}} \left[\text{KL}(\pitdist(\cdot | s, \theta_t) || \pidist(\cdot | s, \theta))  \right] \leq \delta$, we observe that~\cref{eq:fmd-softmax-kl-practical-2} involves the logarithm of $p^\pi$, which can be interpreted as a form of soft clipping due to the narrower range of the log ratio. Additionally, when the policy is modeled by a deep network with a final softmax layer, this leads to an objective concave in the last layer, which is in general easier to optimize than the original TRPO objective. Unlike TRPO, the proposed update enforces the proximity between policies via a regularization rather than a constraint. This modification has been recently found to be beneficial~\citep{lazic2021optimization}. Finally, the parameter $\delta$ in TRPO is a hyper-parameter that needs to be tuned. In contrast, the regularization strength $\nicefrac{1}{\eta}$ in proposed update can be determined theoretically (\cref{prop:eta_softmax}). 
 
\section{EXPERIMENTAL EVALUATION}
\label{sec:experiments}
\begin{figure*}[t]
    \centering
    \includegraphics[width=\textwidth]{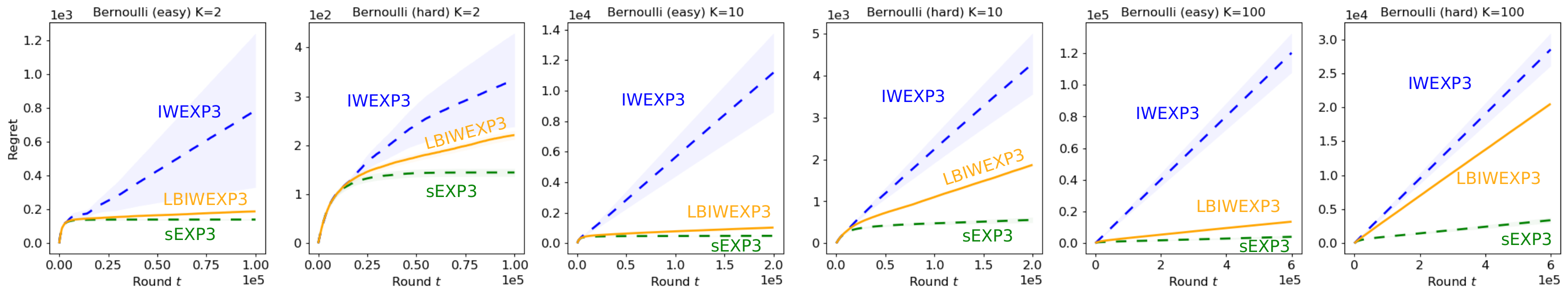}
    \caption{\looseness=-1
    Comparing the average regret over 50 runs for two variants of EXP3 -- with standard importance~weights (IWEXP3) or loss-based importance weights (LBIWEXP3) to that of sEXP3. Both algorithms use a tuned step-size equal to $0.005$. We observe that sEXP3 consistently achieves lower regret. }
    \label{fig:sppo_bandits}
\end{figure*}
While this work focuses on providing a general framework for designing surrogate functions, we explore the behaviour of surrogates instantiated by the softmax functional representation in three different settings. \textbf{First}, to avoid dealing with local maxima of $J$, we explore a multi-armed bandit, where we compare it to the exponential weights algorithm (EXP3)~\citep{auer2002nonstochastic} in~\cref{sec:mab}. The simplicity of the environment  allows us to get a clearer understanding of the behaviour of each algorithm. \textbf{Second}, we set up small-scale RL environments where the surrogates in~\cref{sec:instantiation} can be maximized exactly. We assume access to the exact MDP dynamics and rewards model to focus on the impact of the proposed surrogate, ignoring potential interactions with a critic and avoiding exploration and sampling issues. \textbf{Finally}, we tested the practical performance of \FMAPG/ using a larger-scale experiment on MuJoCo in~\cref{sec:mujoco}. In addition to the increased complexity of the environments, this experiment allows us to explore how the surrogate behaves in the presence of a critic. Since the policies are parameterized as a deep network, the surrogate can only be maximized approximately.  
\vspace{-2ex}
\subsection{Multi-armed Bandit}
\label{sec:mab}
For a stochastic multi-armed bandit problem, we compare EXP3, which corresponds to the single-state, tabular parameterization of \FMAPG/ with the direct representation and the negative entropy mirror map; to softmax EXP3 (\emph{sEXP3}), which uses the softmax parameterization and the logsumexp mirror map. For EXP3, we use the standard importance weighting procedure (denoted as IWEXP3 in the plots) as well as the loss-based variation~\citep{lattimore2020bandit} (denoted as LBIWEXP3). We choose the step-size $\eta$ that achieved the best average final regret for each algorithm over 50 runs (see~\cref{app:bandits} for details). ~\cref{fig:sppo_bandits} shows that sEXP3 consistently achieves lower regret than the EXP3 variants, regardless of the number of arms (2, 10, 100) and the problem difficulty determined by the action gap.

\subsection{Tabular MDP}
\label{sec:mdp}
We use two tabular environments: CliffWorld~\citep{sutton18book} and DeepSeaTreasure~\citep{osband2019behaviour}, and a tabular softmax policy parameterization (one parameter for each state and action). We study the performance of four algorithms, two of which are instantiated by the \FMAPG/ framework -- (i) sMDPO (maximizing the objective given in Eq. \eqref{eq:fmd-softmax-kl-practical}), (ii) MDPO (objective given in \eqref{eq:fmd-softmax-statewise} with a negative entropy mirror map), and two commonly used PG methods -- (iii) PPO~\citep{schulman2017proximal} and (iv) TRPO~\citep{schulman2015trust}. Of these, sMDPO and MDPO have two hyper-parameters: $\eta$ (outer loop stepsize) and $\alpha$ (the inner loop step-size); PPO has two hyper-parameters, $\epsilon$ (clipping factor) and $\alpha$ (the inner-loop step-size) whereas TRPO has a single hyper-parameter $\delta$, the magnitude of the KL-constraint.
For all the algorithms, we use the true action-value functions. The complete experimental setup, implementation details, and additional experiments are in~\cref{app:tabular_experiments} and \ref{app:tabular_derivations}.

\begin{figure}[hbp]
    \centering
    \includegraphics[scale=0.47]{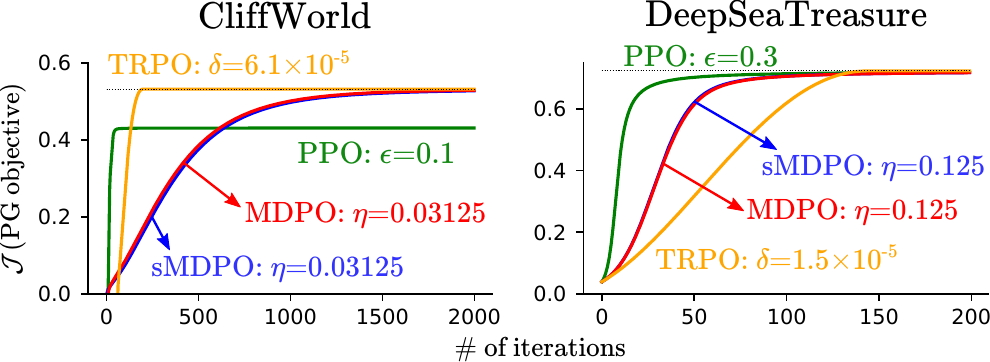}
    \caption{Comparing PG algorithms on CliffWorld and DeepSeaTreasure environments for $100$ inner-loop updates and best set of hyper-parameters. 
    \label{fig:learning_curves}}
\end{figure}
\begin{figure*}[!ht]
    \centering
    \includegraphics[width=0.95\textwidth]{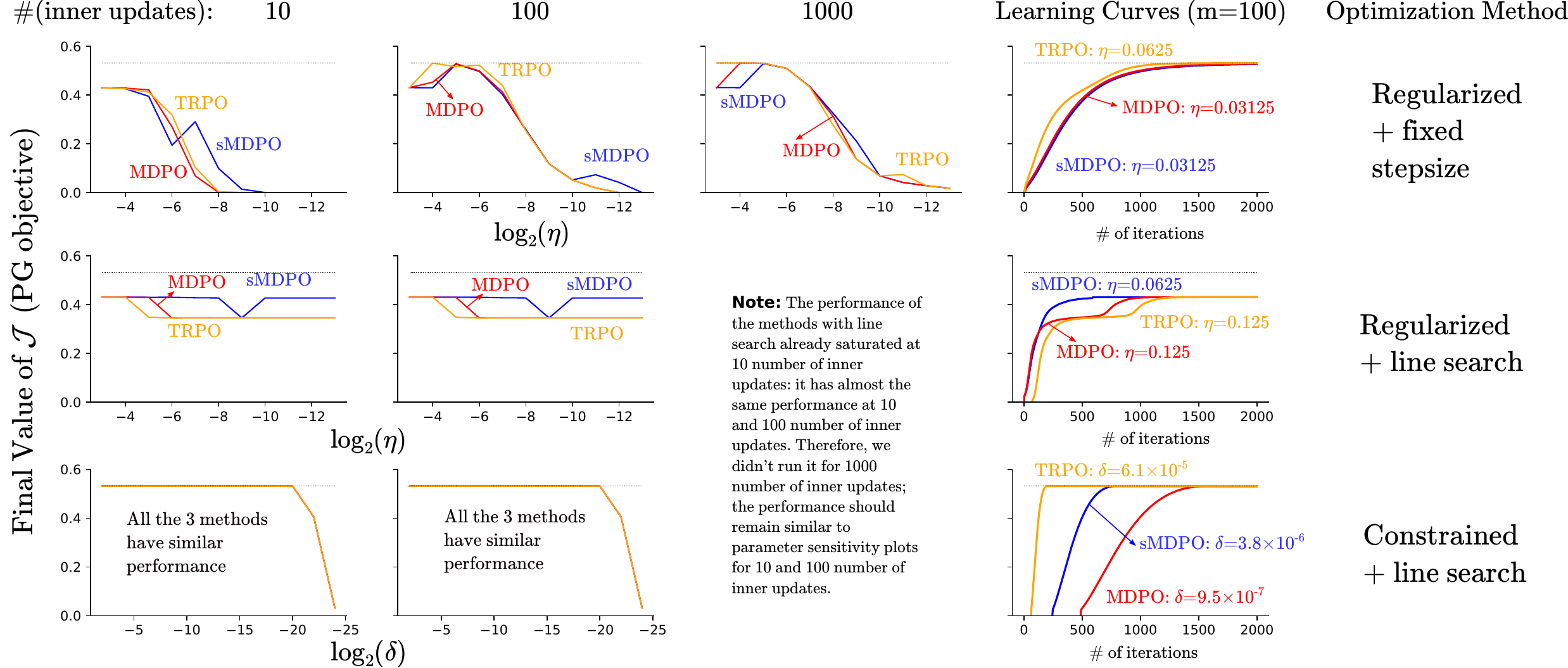}
    \caption{Parameter sensitivity for sMDPO, MDPO and TRPO on  CliffWorld for $2000$ environment interactions with different number of inner loop updates (first 3 columns) and different algorithmic choices (rows) (see~\cref{app:tabular_experiments} for exact expressions). For each plot, the X-axis shows the sensitivity towards the corresponding hyper-parameter (the other hyper-parameters are set to best-tuned values). The first row shows the regularized (with parameter $\eta$) variants (default variant of sMDPO and MDPO used in~\cref{fig:learning_curves}). The second row also shows the regularized variants but uses a line-search to set the step-size for each inner-loop. Instead of enforcing the proximity between consecutive policies via regularization, the variants in the third row use a constraint with parameter $\delta$ (default variant of TRPO used in~\cref{fig:learning_curves}). For each row, the fourth column shows the algorithm performance vs the number of environment interactions. Black lines correspond to the value of the optimal policy.  
    \label{fig:cliffworld_sensitivity_plots}} 
\end{figure*}
\begin{figure*}[!ht]
    \centering
    \includegraphics[width=0.95\textwidth]{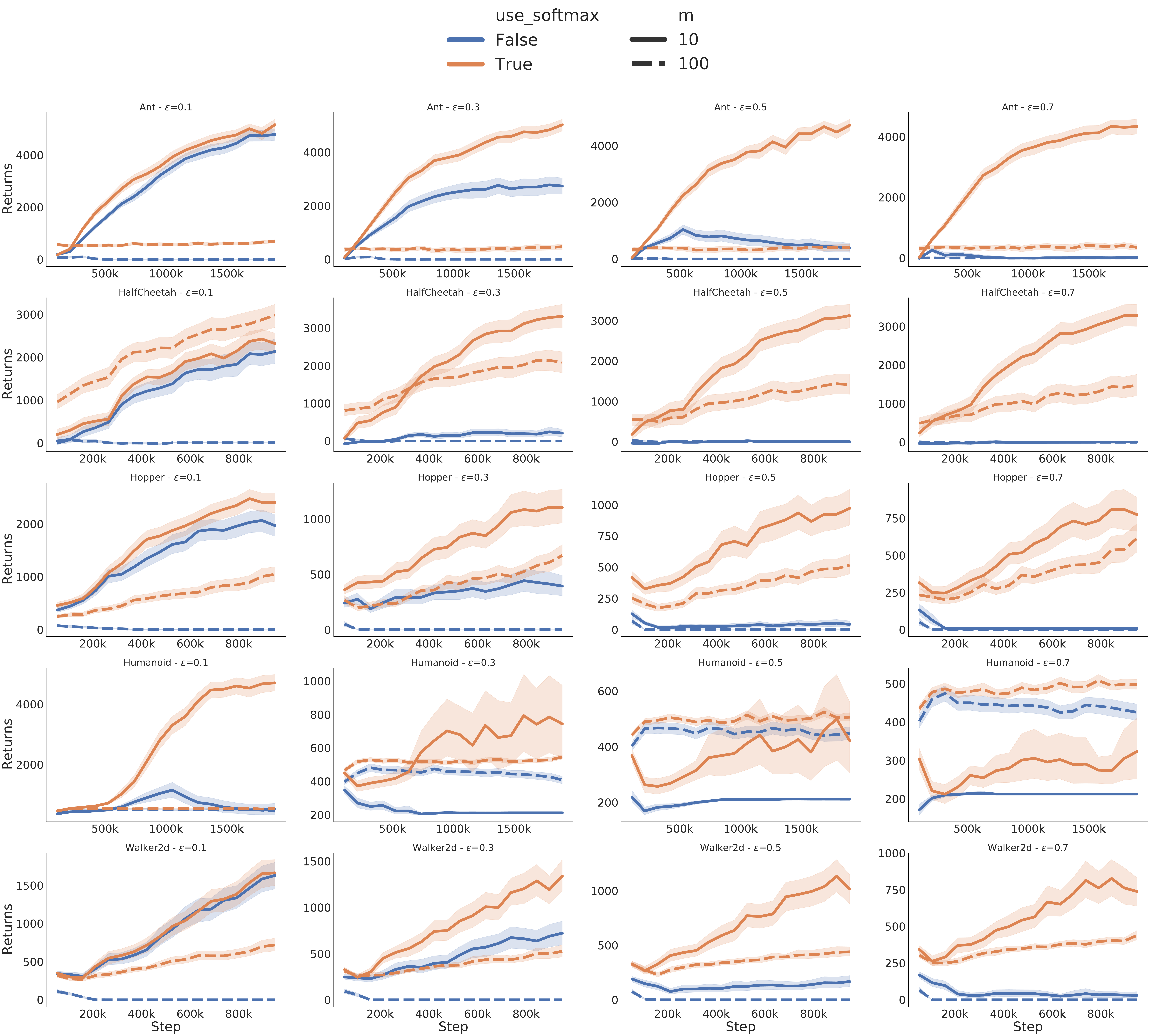}
    \caption{Average return and 95\% confidence intervals (over 180 runs) for {\color{blue}PPO} and {\color{orange}sPPO} on 5 environments rows) and for four different clipping values (columns). sPPO is more robust to large values of clipping, even more so when the number of updates in the inner loop grows (linestyle).
    \label{fig:sppo_mujoco_clean}}
\end{figure*}
\cref{fig:learning_curves} shows the algorithm performance with the number of outer-loops (interactions with the environment) for $m=100$ inner-loop updates. We show the performance for the best set of hyper-parameters for each algorithm and environment. We observe that (i) with exact computation of action-value functions, MDPO and sMDPO have similar performance, and (ii) for both environments, sMDPO, MDPO and TRPO are able to reach the performance of the optimal policy, whereas PPO (with the best hyper-parameter) converges to a sub-optimal policy for CliffWorld. For both sMDPO and MDPO, the theoretically derived step-sizes in~\cref{prop:eta_direct} and~\cref{prop:eta_softmax} are much smaller than the best tuned step-sizes (see~\cref{app:tabular_discussion} for exact calculations). Using theoretically derived step-sizes result in slow (but monotonic) convergence, verifying~\cref{thm:improvement-parametric}. Our results show that sMDPO and MDPO are competitive with popular PG algorithms, and demonstrate the effectiveness of \FMAPG/ in designing theoretically sound and practical PG methods.   

\textbf{Ablation Study:} In~\cref{fig:cliffworld_sensitivity_plots}, we study the effect of different algorithmic choices and sensitivity towards the corresponding hyper-parameter (see the caption for details) for sMDPO, MDPO, and TRPO for CliffWorld (DeepSeaTreasure results in~\cref{app:tabular_experiments}). We observe that (i) increasing the number of inner-loops (marginally) improves the performance of each method, demonstrating the effect of data reuse (ii) in the first row, all methods perform worse as the regularization increases from left to right, and the regularized variant of TRPO~\citep{lazic2021optimization} has similar performance as sMDPO and MDPO, (iii) in the second row, using a line-search for the inner-loop makes all methods more robust to $\eta$, but the aggressive (using large step-sizes) inner-loop updates can result in convergence to a sub-optimal policy, (iv) in the third row, the constrained variants of all methods are quite robust to the constraint hyper-parameter $\delta$, with all methods converging to the optimal policy. Hence, for each method, using a constraint to enforce proximity between consecutive policies can result in superior performance over its regularized counterpart (with or without line-search).

\vspace{-2ex}
\subsection{Large-scale Continuous Control Tasks}
\label{sec:mujoco}
Since PPO~\citep{schulman2017proximal} requires clipping the importance sampling ratio, in order to make the resulting algorithm similar to PPO for ease of implementation, we included clipping with the surrogate function instantiated by \FMAPG/. In particular, we modify~\cref{eq:fmd-softmax-kl-practical} and the resulting surrogate given by: 
\begin{align*}
\ell_t^{\pi, \phi, \eta}(\theta) & = \E_{(s,a) \sim \mu^{\pit}} \bigg[ A^{\pi_t}(s, a) \\ & \times \log\left(\text{clip}\left( \frac{\pidist(a | s, \theta)}{\pidist(a | s, \theta_t)}, \frac{1}{1 + \epsilon}, 1 + \epsilon \right)\right)\bigg], 
\end{align*}
We denote the above surrogate function and the resulting algorithm as \emph{sPPO}. We investigate its performance on five continuous control environments from the OpenAI Gym \citep{gym}: Hopper-v1, Walker2d-v1, HalfCheetah-v1, Ant-v1, and Humanoid-v1. As a baseline, we use the PPO implementation from \citet{andrychowicz2021matters} with their standard configuration and default hyperparameters values. We implement sPPO by adding a binary flag (\texttt{use\_softmax}). We re-emphasize that both algorithms use a critic and that the hyper-parameters of the critic are tuned using PPO to avoid favoring our framework.

We investigate the differences between PPO and sPPO by training 180 different policies for each environment and all combinations of $\texttt{use\_softmax}\in\{\texttt{True}, \texttt{False}\}$, $m\in\{10, 100\}$ and the importance weight capping value $\epsilon \in \{0.1, 0.3, 0.5, 0.7\}$ (a total compute of 1400 days with TPUv2). We evaluate each policy $18$ times during training, using the action with largest probability rather than a sample. We compute the average return and 95\% confidence intervals for each of the settings. The results are presented in~\cref{fig:sppo_mujoco_clean}, where we see that sPPO outperforms PPO across all environments. Furthermore, we see that the difference is more pronounced when the number of iterations $m$ in the inner loop is increased (linestyles) or when less capping is used (columns). 
In~\cref{app:mujoco}, we show additional results but with learning rate decay and gradient clipping disabled, two commonly used techniques to stabilize PPO training~\citep{engstrom2019implementation}. In this setting, sPPO only suffers a mild degradation while PPO fails completely,  confirming sPPO's additional robustness.

% This section demonstrate that . They are also somewhat robust to the hyperparameters; in particular we found that the both algorithms, for all the value of $\eta$ suggested by theory, were able to converge to a locally optimal policy. The sensitivity plots also highlight the importance of data reuse (multiple off-policy type of PG updates): as we increased the number of inner updates, the performance of all the methods improved. Our experiments thus suggest that 

\vspace{-2ex}
\section{CONCLUSION}
\label{sec:conclusion}
\vspace{-2ex}
We proposed \FMAPG/, a general framework to design computationally efficient policy gradient methods. By disentangling the functional representation of a policy from its parameterization, we unified different PG perspectives, recovering several existing algorithms and implementation heuristics in a principled manner. By using the appropriate theoretically-determined hyper-parameters, \FMAPG/ guarantees policy improvement (and hence convergence to a stationary point) for the resulting PG method, even with arbitrarily complex policy parameterizations and for arbitrary number of inner-loop steps. We demonstrated that \FMAPG/ enables the design of new, improved surrogate functions that can lead to improved empirical results. We believe that our framework will further enable the systematic design of sample-efficient PG methods. 

Our theoretical results assume the exact computation of the action-value and advantage functions, and are thus limited in practice. In the future, we aim to handle sampling errors and extend these results to the actor-critic framework.

% Furthermore, we hope to use the \FMAPG/ framework to develop other theoretically-principled PG methods.    
% , as testified by the  softmax representation in various settings. 

\section{Acknowledgements}
We would like to thank Veronica Chelu for suggesting the use of the log-sum-exp mirror map in Section 5. Nicolas Le Roux and Marlos C. Machado are funded by a CIFAR chair. Sharan Vaswani and Shivam Garg gratefully acknowledge support from Csaba Szepesv\'{a}ri during the duration of this project.
\bibliographystyle{apalike}
\bibliography{ref}

\clearpage
\appendix
\thispagestyle{empty}

% For one-column format, uncomment the following:
\onecolumn \makesupplementtitle

\section*{Organization of the Appendix}
\begin{itemize}
   
    \item[\ref{app:svg}] \nameref{app:svg}
     
    \item[\ref{app:proofs-fmapg}] \nameref{app:proofs-fmapg}
    
    \item[\ref{app:proofs-instantiation}] \nameref{app:proofs-instantiation}
    
    \item[\ref{app:bandits}] \nameref{app:bandits}
    
    \item[\ref{app:tabular_experiments}] \nameref{app:tabular_experiments}
    
    \item[\ref{app:tabular_derivations}] \nameref{app:tabular_derivations}
    
    \item[\ref{app:mujoco}] \nameref{app:mujoco}

\end{itemize}

\section{Handling stochastic value gradients}
\label{app:svg}
Thus far we have worked with the original formulation of policy gradients where a policy is a distribution over actions given states. An alternative approach is that taken by stochastic value gradients~\citep{heess2015learning}, that rely on the reparametrization trick. In this case, a policy is not represented by a distribution over actions but rather by a set of actions. Formally, if $\varepsilon$ are random variables drawn from a fixed distribution $\nu$, then policy $\pi$ is a deterministic map from $\mathscr{S} \times \nu \rightarrow \mathscr{A}$. This corresponds to the functional representation of the policy. The action $a$ chosen by $\pi$ in state $s$ (when fixing the random variable $\epsilon = \varepsilon$) is represented as $\pi(s, \epsilon)$ and
\begin{align}
    J(\pi) &= \sum_s d^\pi(s) \int_\varepsilon \nu(\varepsilon) \, r(s, \pi(s, \varepsilon)) \, d\varepsilon \label{eq:svg-obj}
\end{align}
and~\citet{silver2014deterministic} showed that $\displaystyle \frac{\partial J(\pi)}{\partial \pi(s,\epsilon)} = d^\pi(s) \nabla_a Q^\pi(s, a)\big|_{a = \pi(s, \epsilon)}$.

If the policy $\pi$ is parameterized by model $f$ with parameters $\theta$, then $\pi(s, \epsilon) = f(\theta, s, \epsilon)$. If $f(\theta_t, \epsilon)$ and $f(\theta, \epsilon)$ are $S$-dimensional vectors, then~\cref{eq:fmd-update} is given as
\begin{align}
\theta_{t+1} &= \arg\min \sE_{\epsilon \sim \nu} \left[-\sum_s d^{\pi_t}(s) f(\theta, s, \epsilon) \nabla_a Q^{\pi_t}(s, a)\big|_{a = f(\theta_t, s,\epsilon)} + \frac{1}{\eta}D_\phi(f(\theta, \epsilon), f(\theta_t, \epsilon))\right] \; .   
\label{eq:fmd-svg}
\end{align}
Similar to~\cref{sec:direct,sec:softmax}, we will use a mirror map that decomposes across states. Specifically, we choose $\breg{\pi}{\mu} = \sum_{s \in \mathscr{S}} d^\pit(s) \, \norm{\pi(s) - \mu(s)}^{2}$. With this choice,~\cref{eq:fmd-svg} can be written as:
\begin{align}
\theta_{t+1} &= \argmax \left[ \sE_{s \sim d^\pit} \left[\sE_{\epsilon \sim \nu} \left[f(\theta, s, \epsilon) \nabla_a Q^{\pi_t}(s, a)\big|_{a = f(\theta_t, s,\epsilon)} - \frac{1}{\eta} \, \norm{f(\theta, \epsilon) - f(\theta_t, \epsilon)}^{2} \right] \right] \right]   
\label{eq:fmd-svg-stagewise}    
\end{align}
This formulation is similar to Eq~(15) of~\citep{silver2014deterministic}, with $Q^{\pi_t}$ instead of $Q^\pi$. Additionally, while the authors justified the off-policy approach with an approximation, our formulation offers guarantees provided $\eta$ satisfies the condition of Proposition~\cref{prop:eta_smoothness}. 

\section{Proofs for~\texorpdfstring{\cref{sec:guarantees}}{}}
\label{app:proofs-fmapg}

\myquote{\begin{thmbox}
\smooth*
\end{thmbox}
}
\begin{proof}
\begin{align*}
J(\pi) - \ell_t^{\pi, \phi, \eta}(\pi) &= J(\pi) - J(\pit) - \inner{\pi-\pi_t}{\nabla_\pi J(\pit)} + \frac{1}{\eta}D_\phi(\pi, \pit)\\
    &= J(\pi) - J(\pit) - \inner{\pi-\pi_t}{\nabla_\pi J(\pit)} + \frac{1}{\eta}\left(\phi(\pi) -\phi(\pit) - \inner{\nabla_\pi\phi(\pit)}{\pi - \pit}\right)\\
    &= \left(J + \frac{1}{\eta}\phi\right)(\pi) - \left(J + \frac{1}{\eta}\phi\right)(\pit) - \inner{\pi - \pi_t}{\nabla_\pi \left(J + \frac{1}{\eta}\phi\right)(\pit)} \; .
\end{align*}
The last equation is positive for all $\pi$ and all $\pit$ if and only if $J + \frac{1}{\eta}\phi$ is convex.
\end{proof}

\myquote{\begin{thmbox}
\improvementparametric*
\end{thmbox}
}
\begin{proof}
Using the update in~\cref{alg:generic} with $\alpha = \frac{1}{\beta}$ and the $\beta$-smoothness of $\ell_t(\omega)$,  for all $k \in [m-1]$, 
\begin{align*}
\ell_t(\omega_{k+1}) & \geq \ell_{t}(\omega_{k}) + \frac{1}{2 \beta} \normsq{\nabla \ell_{t}(\omega_k)} \\
\intertext{After $m$ steps,}
\ell_{t}(\omega_{m}) & \geq \ell_{t}(\omega_{0}) + \frac{1}{2 \beta} \sum_{k = 0}^{m-1} \normsq{\nabla \ell_{t}(\omega_k)} \\
\intertext{Since $\theta_{t+1} = \omega_{m}$ and $\omega_{0} = \theta_{t}$ in~\cref{alg:generic},}
\implies \ell_{t}(\theta_{t+1}) & \geq \ell_{t}(\theta_t) + \frac{1}{2 \beta} \normsq{\nabla \ell_{t}(\theta_t)} + \sum_{k = 1}^{m-1} \normsq{\nabla \ell_{t}(\omega_k)} \\
\end{align*}
Note that $J(\pit) = \ell_t(\theta_t)$ and if $\eta$ satisfies~\cref{prop:eta_smoothness}, then $J(\pitt) \geq \ell_t(\theta_{t+1})$. Using these relations, 
\begin{align*}
J(\pitt) \geq J(\pit) + \underbrace{\frac{1}{2 \beta} \normsq{\nabla \ell_{t}(\theta_t)} + \sum_{k = 1}^{m-1} \normsq{\nabla \ell_{t}(\omega_k)}}_{\text{+ve}} \implies J(\pitt) \geq J(\pit).
\end{align*}
\end{proof}
\section{Proofs for~\texorpdfstring{\cref{sec:instantiation}}{}}
\label{app:proofs-instantiation}
In this section, we first prove the equivalence of the formulations in terms of the logits and in terms of $\log \pi$.
\begin{lemma}
\label{lemma:divergence_equivalence}
Let
\begin{align}
    \phi(z) &= \log \left(\sum_a \exp(z(a))\right)\\
    p^\pi(a) &= \frac{\exp(z(a))}{\sum_{a'} \exp(z(a'))} \; .
\end{align}
Then
\begin{align}
    D_\phi(z, z') &= KL(p^{\pi'} || p^\pi) \; .
\end{align}
where $p^\pi$ and $p^{\pi'}$ use $z$ and $z'$ respectively.
\end{lemma}
\begin{proof}
\begin{align*}
    D_\phi(z, z') &= \log \left(\sum_a \exp(z(a))\right) - \log \left(\sum_a \exp(z'(a))\right) - \frac{\sum_a \exp(z'(a)) (z(a) - z'(a))}{\sum_a \exp(z'(a))}\\
    &= \sum_a p^{\pi'}(a) \left(z(a) - z'(a) + \log \left(\sum_a \exp(z(a))\right) - \log \left(\sum_a \exp(z'(a))\right)\right)\\
    &= \sum_a p^{\pi'}(a) \log \frac{p^\pi(a)}{p^{\pi'}(a)} \; .
\end{align*}
\end{proof}

\begin{thmbox}
\begin{proposition}
\label{prop:fmd-softmax-kl}
\begin{align}
\ell_t^{z^\pi, \phi, \eta}(\theta) &= J(\pit) + E_{(s, a) \sim \mu^{\pi_t}} \left(A^{\pi_t}(s, a) + \frac{1}{\eta}\right)\log \frac{p^\pi(a | s, \theta)}{p^\pit(a | s, \theta)}
\end{align}
\end{proposition}
\end{thmbox}
\begin{proof}
Because $\sum_a p^\pit(a | s) A^\pit(s, a) = 0$, we can shift all values of $z$ by a term that does not depend on $a$ without changing the sum, in particular by $\log \left(\sum_{a'} \exp(z^\pi(a', s |\theta)\right)$. Thus,
\begin{align*}
\ell_t^{z^\pi, \phi, \eta}(\theta) &= J(\pit) + E_{(s, a) \sim \mu^{\pi_t}} A^{\pi_t}(s, a) \left(z^\pi(a, s | \theta) - \log \left(\sum_{a'} \exp(z^\pi(a', s |\theta))\right)\right)\\
&\quad - \frac{1}{\eta}\sum_s d^\pit(s) D_{\phi_z}(z^\pi(\cdot, s | \theta), z^\pi(\cdot, s, \theta_t))\\
&= J(\pit) + E_{(s, a) \sim \mu^{\pi_t}} A^{\pi_t}(s, a) \log p^\pi(a | s, \theta) - \frac{1}{\eta}\sum_s d^\pit(s) D_{\phi_z}(z^\pi(\cdot, s | \theta), z^\pi(\cdot, s| \theta_t))\\
&= J(\pit) + E_{(s, a) \sim \mu^{\pi_t}} A^{\pi_t}(s, a) \log p^\pi(a | s, \theta) - \frac{1}{\eta}\sum_s d^\pit(s) KL((p^{\pi'}(\cdot | s) || p^\pi(\cdot | s))\; ,
\end{align*}
where the last line is obtained using Lemma~\ref{lemma:divergence_equivalence}. Expanding the KL leads to the desired result.
\end{proof}

\myquote{\begin{thmbox}
\improvdirect*
\end{thmbox}
}
\begin{proof}
\citet{agarwal2019optimality} show that, when using the direct parameterization, $J$ is $\left(\frac{2 \gamma |A|}{(1 - \gamma)^3}\right)$-smooth w.r.t. the Euclidean distance. By using the properties of relative smoothness~\citep{lu2018relatively}, if the mirror map $\phi$ is $\mu$-strongly convex w.r.t. Euclidean distance, then $J$ is $L$-smooth with $L = \left(\nicefrac{2 \gamma |A|}{(1 - \gamma)^3 \, \mu}\right)$. Using the fact that negative entropy is $1$-strongly convex w.r.t. the $1$-norm, we can set $\eta = \nicefrac{(1 - \gamma)^3}{2 \gamma |A|}$ in~\cref{eq:fmd-direct-statewise-practical}. 
\end{proof}

To prove the value of $\eta$ guaranteeing improvement for the softmax parameterization, we first need to extend a lower bound result from~\citet{ghosh2020operator}:
\begin{lemma}
\label{proposition:lb_j_mu}
Let us assume that the rewards are lower bounded by $-c$ for some $c \in \R$. Then we have
\begin{align}
J(\pi) &\geq J(\pit) + E_{(s,a) \sim \mu^\pit} \left[\left(Q^\pit(s, a) + \frac{c}{1-\gamma}\right)\log\frac{p^\pi(a | s)}{p^\pit(a|s)}\right] \; . \label{eq:ghosh_extension}
\end{align}
\end{lemma}
\begin{proof}
Let us define the function $J_\nu$ for a policy $\nu$ as
\begin{align*}
    J_\nu(\pi)  &= \sum_{h=0}^{+\infty} \gamma^h \int_{\tau_h} (r(s_h, a_h) + c)\left(1 + \log\frac{\pi_h(\tau_h)}{\nu_h(\tau_h)}\right)\nu_h(\tau_h) \; d\tau_h - \frac{c}{1-\gamma}\; ,
\end{align*}
where $\tau_h$ is a trajectory of length $h$ that is a prefix of a full trajectory $\tau$ and $\pi_h$ is the policy restricted to trajectories of length $h$. We first show that it satisfies $J_\nu(\pi) \leq J(\pi)$ for any $\nu$ and any $\pi$ such that the support of $\nu$ covers that of $\pi$.

Indeed, we can rewrite
\begin{align*}
    J(\pi)  &= \int_\tau \left(R(\tau) + \frac{c}{1-\gamma}\right) \pi(\tau) \; d\tau - \frac{c}{1-\gamma}\\
            &= \int_\tau \left(\sum_h \gamma^h (r(a_h, s_h) + c)\right)\pi(\tau) \; d\tau  - \frac{c}{1-\gamma} \tag{using $\sum_h \gamma^h c = c/(1-\gamma)$}\\
            &= \sum_h \gamma^h \int_{\tau_h} (r(a_h, s_h) + c)\pi_h(\tau_h) \; d\tau_h   - \frac{c}{1-\gamma} \; ,
\end{align*}
where the last line is obtained by marginalizing over steps $h+1, \ldots, +\infty$ for all $h$ and all trajectories $\tau$.
Because $r(a_h, s_h) + c$ is positive, as the rewards are lower bounded by $-c$, we have
\begin{align*}
    J(\pi)  &= \sum_h \gamma^h \int_{\tau_h} (r(a_h, s_h) + c)\frac{\pi_h(\tau_h)}{\nu_h(\tau_h)}\nu_h(\tau_h) \; d\tau_h - \frac{c}{1-\gamma}\\
            &\geq \sum_h \gamma^h \int_{\tau_h} (r(a_h, s_h) + c)\left(1 + \log\frac{\pi_h(\tau_h)}{\nu_h(\tau_h)}\right)\nu_h(\tau_h) \; d\tau_h - \frac{c}{1-\gamma} \tag{using $x \geq 1 + \log x$}\\
            &= J_\nu(\pi) \; .
\end{align*}

Let us denote $J^{SA}_\nu$ the right-hand side of~\cref{eq:ghosh_extension}, i.e.:
\begin{align*}
    J^{SA}_\nu(\pi) &= J(\nu) + E_{(s,a) \sim \mu^\nu} \left[\left(Q^\nu(s, a) + \frac{c}{1-\gamma}\right)\log\frac{p^\pi(a | s)}{p^\nu(a|s)}\right] \; .
\end{align*}

We now prove that $J_\nu$ has the same gradient as $J^{SA}_\nu$:
\begin{align*}
\nabla_\theta J_\nu(\pi)    &= \nabla_\theta \left(\sum_h \gamma^h \int_{\tau_h} (r(a_h, s_h) + c)\left(1 + \log\frac{\pi_h(\tau_h)}{\nu_h(\tau_h)}\right)\nu_h(\tau_h) \; d\tau_h\right)\\
&= \nabla_\theta \left(\sum_h \gamma^h \int_{\tau_h} (r(a_h, s_h) + c)\log\pi_h(\tau_h)\nu_h(\tau_h)\right) \; d\tau_h\\
&= \sum_h \gamma^h \int_{\tau_h} (r(a_h, s_h) + c)\nabla_\theta\log\pi_h(\tau_h)\nu_h(\tau_h) \; d\tau_h\; ,
\end{align*}
where all terms independent of $\theta$ were moved outside of the gradient. As the log probability of a trajectory decomposes into a sum of the probabilities of actions given states and of the transition probabilities, and as the latter are independent of $\theta$, we get
\begin{align*}
\nabla_\theta J_\nu(\pi)&= \sum_h \gamma^h \int_{\tau_h} (r(a_h, s_h) + c)\nabla_\theta\log\pi_h(\tau_h)\nu_h(\tau_h) \; d\tau_h\\
&=\sum_h \gamma^h \int_{\tau_h} (r(a_h, s_h) + c)\left(\sum_{h'} \nabla_\theta\log p^\pi(a_{h'} | s_{h'})\right)\nu_h(\tau_h) \; d\tau_h\\
&= \int_{\tau} \sum_{h'} \nabla_\theta\log p^\pi(a_{h'} | s_{h'}) \left(\sum_{h=h'}^{+\infty} \gamma^h (r(a_h, s_h) + c)\right)\nu(\tau) \; d\tau\; .
\end{align*}
But
\begin{align*}
    \sum_{h=h'}^{+\infty} \gamma^h (r(a_h, s_h) + c) &= \gamma^{h'}\left(Q^\nu(s, a) + \frac{c}{1-\gamma}\right)\\
    \int_\tau \nu(\tau)d\tau 1_{a_{h'} = a}1_{s_{h'} = s} &= d^{h'}_\nu(s) \nu(a | s)\;,
\end{align*}
with $d^{h'}_\nu(s)$ the undiscounted probability of reaching state $s$ at timestep $h'$. Hence, we have
\begin{align*}
\nabla_\theta J_\nu(\pi)&= \int_{\tau} \sum_{h'} \nabla_\theta\log p^\pi(a_{h'} | s_{h'}) \left(\sum_{h=h'}^{+\infty} \gamma^h (r(a_h, s_h) + c)\right)\nu(\tau) \; d\tau\\
&= \sum_{h'} \sum_s \sum_a \nabla_\theta\log p^\pi(a | s) d^{h'}_\nu(s) \nu(a | s)\gamma^{h'}\left(Q^\nu(s, a) + \frac{c}{1-\gamma}\right)\\
&= \sum_{h'} \gamma^{h'} \sum_s d^{h'}_\nu(s) \sum_a \nabla_\theta\log p^\pi(a | s)  \nu(a | s)\left(Q^\nu(s, a) + \frac{c}{1-\gamma}\right)\\
&= \sum_s d^\nu(s)\sum_a \left(Q^\nu(s, a) + \frac{c}{1-\gamma}\right)\nu(a | s) \nabla_\theta\log p^\pi(a | s)\\
&= \nabla_\theta\left(\sum_s d^\nu(s)\sum_a \left(Q^\nu(s, a) + \frac{c}{1-\gamma}\right)\nu(a | s) \log p^\pi(a | s)\right)\\
&= \nabla_\theta\left(J(\nu) + E_{(s,a) \sim \mu^\nu} \left[\left(Q^\nu(s, a) + \frac{c}{1-\gamma}\right)\log\frac{p^\pi(a | s)}{p^\nu(a|s)}\right]\right)\\
&=\nabla_\theta J^{SA}_\nu(\pi)\; ,
\end{align*}
with $d^\nu(s)$ the unnormalized probability of $s$ under the \emph{discounted} stationary distribution.

Because $J_\nu$ and $J^{SA}_\nu$ have the same gradient, they differ by a constant, i.e. $J^{SA}_\nu = J_\nu + C$ for some $C$. But we also know that $J_\nu(\nu) = J(\nu)$, which means that
\begin{align*}
    C &= J^{SA}_\nu(\nu) - J_\nu(\nu)\\
    &= J^{SA}_\nu(\nu) - J(\nu)\\
    &= E_{(s,a) \sim \mu^\nu} \left[\left(Q^\nu(s, a) + \frac{c}{1-\gamma}\right)\log\frac{p^\nu(a | s)}{p^\nu(a|s)}\right]\\
    &= 0 \; .
\end{align*}

Hence, $J_\nu = J^{SA}_\nu$ and, becomes $J_\nu$ is a lower bound of $J$, we have
\begin{align}
    J(\pi) &\geq J(\nu) + \sum_s d^\nu(s)\sum_a \left(Q^\nu(s, a) + \frac{c}{1-\gamma}\right) p^\nu(a | s) \log\frac{ p^\pi(a | s)}{p^\nu(a|s)} \; .
\end{align}
Setting $\nu = \pit$ concludes the proof. 
\end{proof}

\myquote{\begin{thmbox}
\improvsoftmax*
\end{thmbox}
}
\begin{proof}
Assume
\begin{align}
    \eta &= \frac{1-\gamma}{r_m-r_l} \; .
\end{align}
We know from~\cref{prop:fmd-softmax-kl} that
\begin{align*}
\ell_t^{z^\pi, \phi, \eta}(\theta) &\leq J(\pit) + E_{(s, a) \sim \mu^{\pi_t}} \left(A^{\pi_t}(s, a) + \frac{1}{\eta}\right)\log \frac{p^\pi(a | s, \theta)}{p^\pit(a | s, \theta)}\; .
\end{align*}

Since the rewards are between $r_l$ and $r_m$, we have
\begin{align*}
\ell_t^{z^\pi, \phi, \eta}(\pi) &\leq J(\pit) + E_{(s, a) \sim \mu^{\pi_t}} \left[\left(A^{\pi_t}(s, a) + \frac{1}{\eta}\right)\log \frac{p^{\pi}(a | s)}{p^{\pit}(a | s)}\right]\\
&= J(\pit) + E_{(s, a) \sim \mu^{\pi_t}} \left[\left(A^{\pi_t}(s, a) +\frac{r_m-r_l}{1-\gamma}\right)\log \frac{p^{\pi}(a | s)}{p^{\pit}(a | s)}\right]\\
&= J(\pit) + E_{(s, a) \sim \mu^{\pi_t}} \left[\left(A^{\pi_t}(s, a) + V^{\pit}(s) + \left(\frac{r_m}{1- \gamma} - V^\pit(s)\right) -\frac{r_l}{1-\gamma}\right)\log \frac{p^{\pi}(a | s)}{p^{\pit}(a | s)}\right]\\
&= J(\pit) + E_{(s, a) \sim \mu^{\pi_t}} \left[\left(Q^{\pi_t}(s, a) -\frac{r_l}{1-\gamma}\right)\log \frac{p^{\pi}(a | s)}{p^{\pit}(a | s)}\right]\\
&\quad - E_{s\sim d^\pit}\left[\left(\frac{r_m}{1- \gamma} - V^\pit(s)\right)KL(p^\pit(\cdot | s)||p^\pi(\cdot | s))\right] \; .
\end{align*}
The last term on the RHS of the last equation is negative. Indeed, because the rewards are less than $r_m$, the value functions are less than $r_m / (1 - \gamma)$  and $r_m/(1-\gamma) - V^\pit(s)$ is positive. As the KL divergences are positive, the product of the two is positive and the whole term is negative because of the minus term.
Thus, we have
\begin{align*}
\ell_t^{z^\pi, \phi, \eta}(\pi) &\leq J(\pit) + E_{(s, a) \sim \mu^{\pi_t}} \left[\left(Q^{\pi_t}(s, a) -\frac{r_l}{1-\gamma}\right)\log \frac{p^{\pi}(a | s)}{p^{\pit}(a | s)}\right]\\
&\leq J(\pi) \tag{by~\cref{proposition:lb_j_mu}} \; .
\end{align*}
Hence, choosing $\eta = \frac{1-\gamma}{r_m-r_l}$ leads to an improvement guarantee. Because our rewards are bounded between 0 and 1, setting $r_m = 1$ and $r_l=0$ gives $\eta = 1-\gamma$.
This concludes the proof.
\end{proof}
\section{Experimental details in the bandit setting}
\label{app:bandits}
In this section, we detail the experimental setup for the bandit experiments in~\cref{sec:mab}.
 
We consider different $K$-armed Bernoulli bandit problems. For sEXP3, we specialising the update rule in~\cref{eq:fmd-softmax-kl-practical} to this multi-armed bandit case yielding: $p^\pitt(a) = p^\pit(a) (1 + \eta A^{\pi_t}(a))$, where $\eta$ needs to be chosen such that the probabilities are always positive. However, the computing the advantage either requires knowledge of the rewards of all arms, or an estimate thereof. Since EXP3 is an adversarial bandit algorithm and does not exploit the stochasticity in the rewards, to ensure a fair comparison, we cannot use such an estimate and thus replace the advantage with the immediate reward, leading to the final sEXP3 update:
\begin{align*}
p^\pitt(a) &= p^\pit(a) (1 + \eta \hat{r}_t(a)) \; ,
\end{align*}
where $\hat{r}_t(a)$ an estimator of the reward $r_t(a)$ obtained at round $t$. 

For sEXP3, if $A_t$ is the action taken at round $t$, then we use the importance weighted estimator $\hat{r}_t(a) = \sI\{A_t=1\} r_t(a)/\pi_t(a)$. For EXP3, we consider both the standard importance weighted estimator (referred to as IWEXP3 in the plots) and the loss based importance weighted estimator (referred to as LBIWEXP3 in the plots) for which $\hat{r}_t(a) =  \sI\{A_t=1\} (1-r_t(a))/\pi_t(a)$. 

Before describing our experimental setup, we emphasize that there are two different sources of randomness in our experiments. First, we have the \emph{environment seed} that controls the mean rewards in the bandit problem. Considering different environment seeds guarantees that our results are not specific to a particular choice of the rewards. Given a specific bandit problem, since EXP3 and sEXP3 are randomized bandit algorithms, there is a stochasticity in the actions chosen. We can use different \emph{agent seeds} to control the algorithm randomness. 

Following the evaluation protocol of~\citep{vaswani2020old}, we consider two classes of bandits with different action gaps (difference in the mean rewards) -- hard instances $(\Delta = 0.5)$ and easy instances $(\Delta = 0.1)$. The mean vector defining a Bernoulli bandit is then sampled entry wise (for each arm) from $\gU(0.5 - \Delta/2, 0.5 + \Delta/2)$. To obtain the plot in~\cref{sec:mab}, we run the experiment for 50 different environment seeds and one agent seed. We evaluated the three algorithms for Bernoulli bandits with  $K \in \{2, 10, 100\}$ arms and the difficulty of the problem, as determined by the action gap.  For each algorithm, we set the step-size via a grid search over $\eta \in \{0.5, 0.05, 0.005, 0.0005, 0.00005 \}$. The plot shows the regret corresponding to the step-size with lowest final average regret. 
\clearpage
\section{Experiments in the tabular setting} \label{app:tabular_experiments}
In this section\footnote{The code implementation for the algorithms and the environment corresponding to experiments presented in this section is available at \url{https://github.com/svmgrg/fma-pg}.}, we study the performance of four different policy gradient (PG) algorithms. Two of these can be directly obtained from the FMA-PG framework: sMDPO (FMA-PG with a softmax policy and log-sum-exp mirror map; see Eq. \ref{eq:fmd-softmax-kl-practical} in the main text) and MDPO (FMA-PG with direct parameterization and a negative entropy mirror map; see Eq. \ref{eq:fmd-softmax-statewise} in the main text). And the other two are the existing popular PG algorithms: TRPO and PPO. 

Further, to better understand the reason behind the performance of each of these methods, in addition to studying the objective functions used by these PG algorithms, we will also consider the impact of the optimization techqniques used to implement them. In particular, we will look at three different variants of sMDPO, MDPO, and TRPO based on whether they use a regularized objective with a fixed step-size (similar to the conventional sMDPO and MDPO), a regularized objective with Armijo line search, or a constrained objective with line search (similar to the conventional TRPO). 

\subsection{Algorithmic Details} \label{sec:tabular_algorithmic_details}
We begin by specifying the different surrogate objectives used by the different algorithms and the two optimization procedures we use for maximizing these objectives. The sMDPO and MDPO algorithms motivated by the FMA-PG framework can be considered as regularized algorithms, which can be summarized by
\begin{equation}
  \max_\theta \; \mathcal{J}_{\text{PG-Alg}} - \frac{1}{\eta} \mathcal{C}_{\text{PG-Alg}}, \label{eq:regularized_program}
\end{equation}
where the terms $\mathcal{J}_{\text{PG-Alg}}$ and $\mathcal{C}_{\text{PG-Alg}}$ are given in Table \ref{table:ablation_study}. One way of solving this objective is by gradient descent using a fixed step-size $\alpha$, as specified in Algorithm \ref{alg:generic}; we call this setting as \texttt{Regularized + fixed step-size}. We can equivalently solve such an unconstrained optimization problem by using an Armijo-style backtracking line search, which we call as \texttt{Regularized + line search}. Note that, we can use this same form to obtain a regularized version of the TRPO algorithm\footnote{For TRPO, this objective is almost the same as PPO with KL penalty (Eq. 8, \citet{schulman2017proximal}) except that PPO uses the advantage function and we used the action value function (which, as we discussed in the caption of Table \ref{table:ablation_study}, doesn't really matter). It is also similar to the objective stated in the TRPO paper (Section 4, \citet{schulman2015trust}) except that this has an average KL divergence instead of the max KL divergence given in the original paper.} as well.

  \begin{table}[!hbt]
  \centering
  \renewcommand{\arraystretch}{1.5}
  \renewcommand{\tabcolsep}{0.2cm}
  \begin{tabular}{c|l|c}
    \textbf{PG Alg.} & \hspace{2cm} \textbf{Objective} $(\mathcal{J})$ & \textbf{Constraint} $(\mathcal{C})$ \\
    \hline \hline
    sMDPO &
    $\sum_s d^{\pi_t}(s) \sum_a p^{\pi_t}(a | s) A^{\pi_t}(s, a) \log \frac{p^{\pi_\theta}(s, a)}{p^{\pi_t}(s, a)}$ &
    $\sum_s d^{\pi_t}(s) \cdot \text{KL}(p^{\pi_t}(\cdot | s) \| p^{\pi_\theta}(\cdot | s))$\\
    \hline
    TRPO &
    $\sum_s d^{\pi_t}(s) \sum_a p^{\pi_t}(a | s) Q^{\pi_t}(s, a) \frac{p^{\pi_\theta}(s, a)}{p^{\pi_t}(s, a)}$ &
    \texttt{(same as above)} \\
    \hline
    MDPO &
    $\sum_s d^{\pi_t}(s) \sum_a p^{\pi_t}(a | s) A^{\pi_t}(s, a) \frac{p^{\pi_\theta}(s, a)}{p^{\pi_t}(s, a)}$ &
    $\sum_s d^{\pi_t}(s) \cdot \text{KL}(p^{\pi_\theta}(\cdot | s) \| p^{\pi_t}(\cdot | s))$
  \end{tabular}
  \caption{The objectives and the constraints corresponding to the different PG algorithms. Note that the objective $\mathcal{J}$ for both TRPO and MDPO is essentially equivalent to each other, since maximizing either of them would lead to the same solution. The reason for this is that the difference between the two objectives is $\sum_s d^{\pi_t} V^{\pi_t}(s)$, which is independent of the policy weight $\theta$).} 
  \label{table:ablation_study}
\end{table}

On the other hand, the conventional TRPO algorithm instead solves a constrained optimization problem given by the equation
\begin{equation}
  \max_\theta \; \mathcal{J}_{\text{PG-Alg}} \quad \text{subject to } \quad \mathcal{C}_{\text{PG-Alg}} \leq \delta, \label{eq:constrained_program}
\end{equation}
with the terms $\mathcal{J}_{\text{PG-Alg}}$ and $\mathcal{C}_{\text{PG-Alg}}$ again given in Table \ref{table:ablation_study}. The regularized program of Eq. \ref{eq:regularized_program} can be considered a ``softer'' version of the constrained program of Eq. \ref{eq:constrained_program}. To solve the constrained optimization problem, we use the exact same process used by the TRPO paper \citep{schulman2015trust}: we use line search to find the maximal step-size that increases the objective value in the direction of maximum ascent while satisfying the constraint; see Section \ref{app:trpo} for details. We call this setting as \texttt{Constrained + line search}. Further, using Eq. \ref{eq:constrained_program}, we can also obtained constrained versions of sMDPO and MDPO. 

The motivation behind considering these three different variants for sMDPO, MDPO, and TRPO is to figure out how much of the performance difference between these algorithms comes from their exact objectives (Table~\ref{table:ablation_study}) and how much of it comes from the optimization techniques employed. We also summarize the gradient of these objectives in Table~\ref{table:ablation_study_grad}. The corresponding gradient derivations for the algorithms (including PPO) are presented in Appendix \ref{app:tabular_derivations}.

\begin{table}[!hbt]
  \centering
  \renewcommand{\arraystretch}{1.5}
  \renewcommand{\tabcolsep}{0.2cm}
  \begin{tabular}{c|l|c}
    \textbf{PG Alg.} & \hspace{1cm} \textbf{Grad. objective} $(\nabla_{\theta(s, a)} \mathcal{J})$ & \textbf{Grad. constraint} $(\nabla_{\theta(s, a)} \mathcal{C})$ \\
    \hline \hline
    sMDPO &
    $d^{\pi_t}(s) p^{\pi_t}(a|s) A^{\pi_t}(s, a)$ &
    $d^{\pi_t}(s) \left[ p^\pi(a | s) - p^{\pi_t}(a|s) \right]$ \\
    \hline
    TRPO &
    $d^{\pi_t}(s) p^\pi(a | s) \left[ Q^{\pi_t}(s, a) - \sum_b p^\pi(b | s) Q^{\pi_t}(s, b) \right]$ &
    \texttt{(same as above)} \\
    \hline
    MDPO &
    $d^{\pi_t}(s) p^\pi(a | s) \left[ A^{\pi_t}(s, a) - \sum_b p^\pi(b|s) A^{\pi_t}(s, b) \right]$ &
    $\begin{array}{c} d^{\pi_t}(s) p^\pi(a | s) \times \\ \left[ \log \frac{p^\pi(a | s)}{p^{\pi_t}(a | s)} - \text{KL}(p^\pi(\cdot | s) \| p^{\pi_t}(\cdot | s)) \right] \end{array}$
  \end{tabular}
  
  \caption{The gradients of the objectives and constraints w.r.t. the policy parameter corresponding to the different PG algorithms. Note that the gradient of the objective for both TRPO and MDPO is exactly equal to each other.}
  \label{table:ablation_study_grad}
\end{table}

\subsection{Empirical Details}
In all our tabular experiments, we assumed full access to the environment dynamics and used the analytically calculated expected gradient updates for all the algorithms, and therefore the results closely follow the theoretical properties of the PG methods. Doing so, essentially made this a study of the optimization properties of the four PG algorithms considered. We use a policy gradient agent with a tabular softmax policy parameterization, and evaluate the algorithms on two tabular episodic environments: CliffWorld (environment description and its properties are discussed in Figure \ref{fig: cliffworld}) and DeepSeaTreasure  (\citet{osband2019behaviour}; with $n=5$, discount factor $\gamma=0.9$, 25 different states, and two actions). 

\begin{figure}[ht]
  \centering
  \includegraphics[scale=0.36]{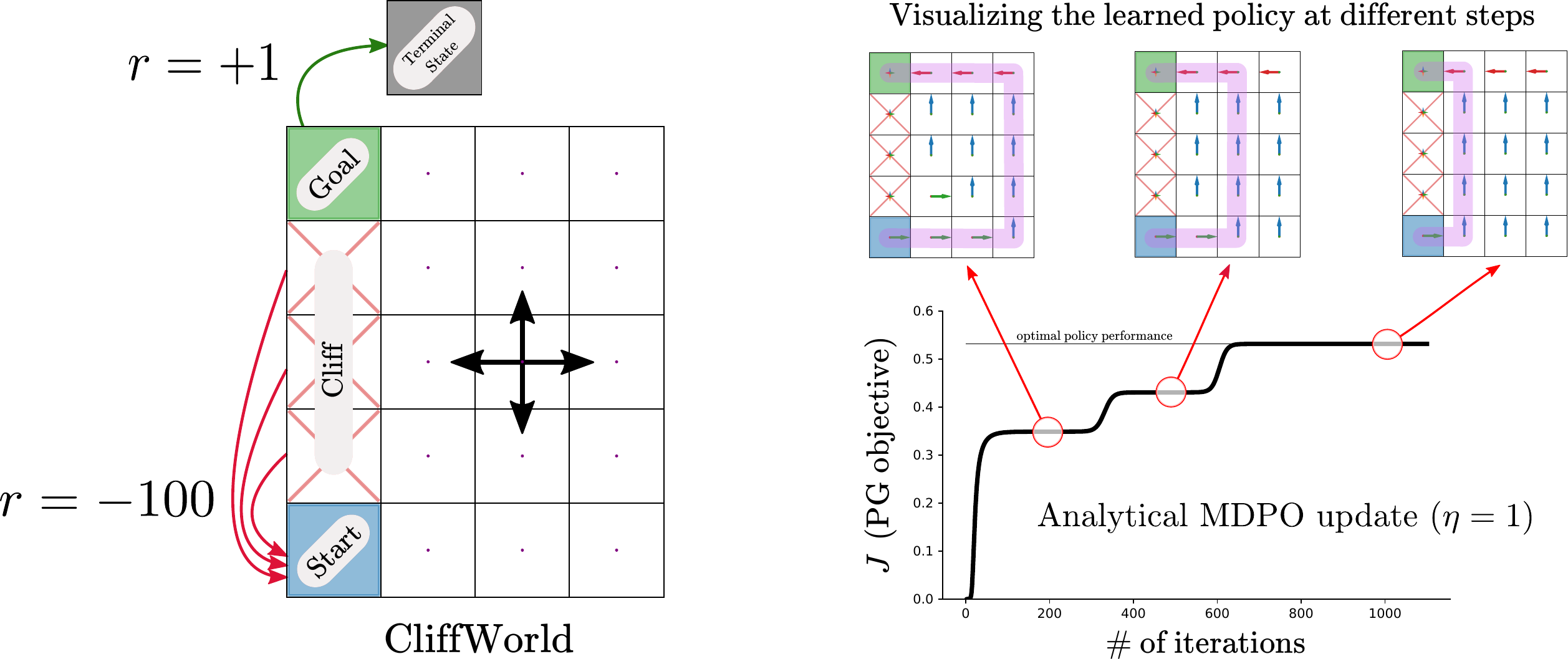}
  \caption{The episodic CliffWorld environment and the learning curve for MDPO on it illustrating three different locally optimal policies. \textbf{(Left)} We consider a variant of the CliffWorld environment (Example 6.6, \citet{sutton18book}) containing 21 different states and four actions per state. The agent starts in the \texttt{Start} state and has four cardinal actions which deterministically move it into the corresponding next state. The objective is to reach the \texttt{Goal} state as quickly as possible. If the agent falls into a state marked by \texttt{Cliff}, any subsequent action taken by it moves it back to the start state and yields a reward of $-100$. Similarly, once in the goal state, any action takes the agent into the terminal state and yields a reward of $+1$. All the other transitions have zero reward and the discount factor is $\gamma = 0.9$. It is easy to see that the optimal policy will have a value of $v^*(s_0) = 0 + \gamma \cdot 0 + \cdots + \gamma^5 \cdot 0 + \gamma^6 \cdot 1 = 0.9^6 = 0.53$. \textbf{(Right)} We show the learning curve for the analytical MDPO update using $\eta = 1$. This curve shows three different locally optimal policies. We later show in our experiments, that the different PG agents often get stuck on one of these policies.}
  \label{fig: cliffworld}
 \end{figure}
  
\begin{figure*}[!tbp]
    \centering
    \includegraphics[width=0.9\textwidth]{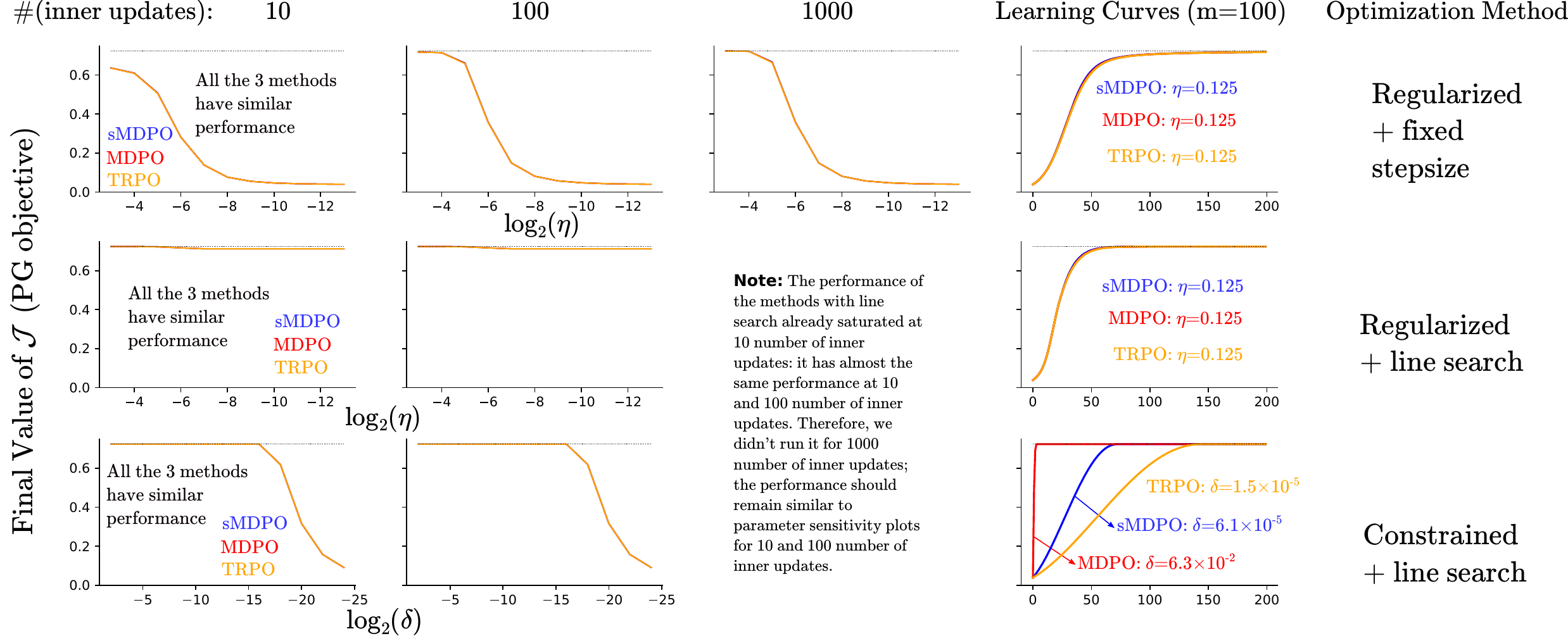}
    \caption{The parameter sensitivity plots for the  PG algorithms on the DeepSeaTreasure environment for different number of inner loop updates. The $x$ axis shows sweep over one parameter of the corresponding PG algorithm. And for each point on the $x$-axis, we chose the best performing second parameter of the algorithm: the inner loop step-size $\alpha$ for the first row, the Armijo constant for the second row, and there is no additional parameter for the last row. The faint black line near the top of each subplot depicts the value of the optimal policy. The last column shows the learning curves for the best performing parameter configuration for each method.
    \label{fig:dst_sensitivity_plots}}
\end{figure*}
 
\textbf{Hyperparameter configurations:} We trained each of the method for 2000 iterations for CliffWorld (200 iterations for DeepSeaTreasure). Each iteration consisted of multiple inner loop updates (we represent this number by $m$); these updates are performed in an off-policy fashion that is typical of all these algorithms (also see Algorithm \ref{alg:generic} in the main paper). We also swept over the relevant parameters of the PG algorithms. For the \texttt{Regularized} variants (both with and without line search) of sMDPO, MDPO, and TRPO, this was $\eta \in \{2^{-13}, 2^{-12}, \ldots, 2^{-1}\}$. For \texttt{fixed step-size} variant of  sMDPO, MDPO, TRPO, and PPO, we swept over the inner loop step-size $\alpha \in \{2^{-13}, 2^{-12}, \ldots, 2^3\}$ for CliffWorld (and $\alpha \in \{2^{-13}, 2^{-12}, \ldots, 2^{-2}\}$ for DeepSeaTreasure). For PPO, we additionally considered the clipping parameter $\epsilon \in \{0.01, 0.1, 0.2, 0.3 \ldots, 0.8, 0.9, 0.99\}$. For the \texttt{Regularized + line search} variant of  sMDPO, MDPO, TRPO, we also considered different Armijo constants in the set $\{0.0, 0.1, 0.3, 0.5, 0.7, 0.9, 0.99\}$, used a decay factor of $0.9$, initialized the maximal step-size to $10.0$ and fixed the warm-start factor to $2.0$. Finally, for the \texttt{Constrained + line search} variant of  sMDPO, MDPO, and TRPO, we swept over the trust region size $\delta \in \{2^{-24}, 2^{-22}, \ldots, 2^{-2}\}$, used a fixed backtracking decay parameter of $0.9$, an analytically obtained maximal step-size (see Appendix \ref{app:trpo}), and an Armijo constant of $0.0$ (i.e. no Armijo line search).

\subsection{Experimental Results}
\textbf{Learning Curves:} We show the learning curves corresponding to the best performing hyperparameters for the four algorithms conventional sMDPO and MDPO (\texttt{Regularized + fixed step-size}), conventional TRPO (\texttt{Constrained + line search}), and PPO in Figure \ref{fig:learning_curves} (main paper). To select the hyperparameters for each setting, we ran sweeps over different configurations and chose the ones that resulted in the best final performance at the end of 2000 iterations for CliffWorld (and 200 iterations for DeepSeaTreasure). From Figure \ref{fig:learning_curves}, we see that for CliffWorld, all the methods except PPO (PPO got stuck in a ``safe'' sub-optimal policy) were able to converge to the optimal policy, and TRPO had the fastest convergence (learned the optimal policy in less than 200 iterations). On the other hand for DeepSeaTreasure, we note that all the methods converged to the optimal policy, with PPO having the fastest convergence and TRPO the slowest. Additionally, we should also mention that the TRPO's update was the costliest (more than two times slower than the rest of the methods) in terms of wall time, likely because of the backtracking from the line-search. 

\textbf{Parameter Sensitivity and Ablation Study:} We show the final performance for sMDPO, MDPO, and TRPO in Figure \ref{fig:cliffworld_sensitivity_plots} (after 2000 iterations for CliffWorld; main paper) and Figure \ref{fig:dst_sensitivity_plots} (after 200 iterations for DeepSeaTreasure). The different rows correspond to the variants \texttt{Regularized + fixed step-size}, \texttt{Regularized + line search}, and \texttt{Constrained + line search} for each of the methods. And the different columns correspond to different number of inner loop updates\footnote{For the \texttt{Constrained + line search} of each method, we observed that the performance saturated after $m = 10$; in particular the sensitivity plots are identical for $m=10$ and $m=100$. Therefore, the performance at $m=1000$ should be exactly equivalent to the performance given at $m=100$, and consequently we skipped running that experiment.}. The last column in each row shows the learning curves for the best performing parameter setting. The $x$-axis on each subplot of the first two rows shows the regularization strength $\eta$. For the \texttt{Regularization + fixed stepsize} variant, we chose the best performing $\alpha$ for each $\eta$, and for \texttt{Regularized + line search} variant, we chose the best performing Armijo constant for each $\eta$. The last row (constrained variant) had only a single parameter, the trust region magnitude, $\delta$ that is shown on the $x$-axis.

From these figures, we see that as the value of $m$ increased, the performance of  the fixed step-size algorithms improved. We also note that adding line search to regularized methods improved their parameter sensitivity to a large extent. Although, for CliffWorld, none of the \texttt{Regularized + line search} variant were able to achieve the optimal policy. We believe that the reason for this is that with warm-start the algorithms started using very large stepsizes (as large as 1000), which lead to an early convergence to a locally optimal policy. To verify this further, we tried running these algorithms (experiments not shown here) without warm start and a maximal stepsize of 1.0; this allowed the methods to achieve the optimal policy for a small range of $\eta$ values, but also made them much more sensitive different values of $\eta$. For the constrained version, we see that all the three algorithms achieved the optimal policy and were generally insensitive to the $\delta$ values. This is likely because the constrained variant used the (near) optimal steepest ascent direction with the maximal stepsize, achieved via a backtracking line search. Finally, we note that for DeepSeaTreasure, all the methods had essentially the same performance and achieved the optimal policy in each case; we attribute this to the simplicity of the environment coupled with access to the true gradient updates.

We also provide the sensitivity plot for PPO for the two environments in Figure \ref{fig:ppo_plots}. We again see that increasing the number of inner loop updates helps the performance of PPO on both the environments. We also note that for no value of the parameters we tested, did PPO achieve the optimal policy on CliffWorld.

\begin{figure*}[!tbp]
    \centering
    \includegraphics[width=0.9\textwidth]{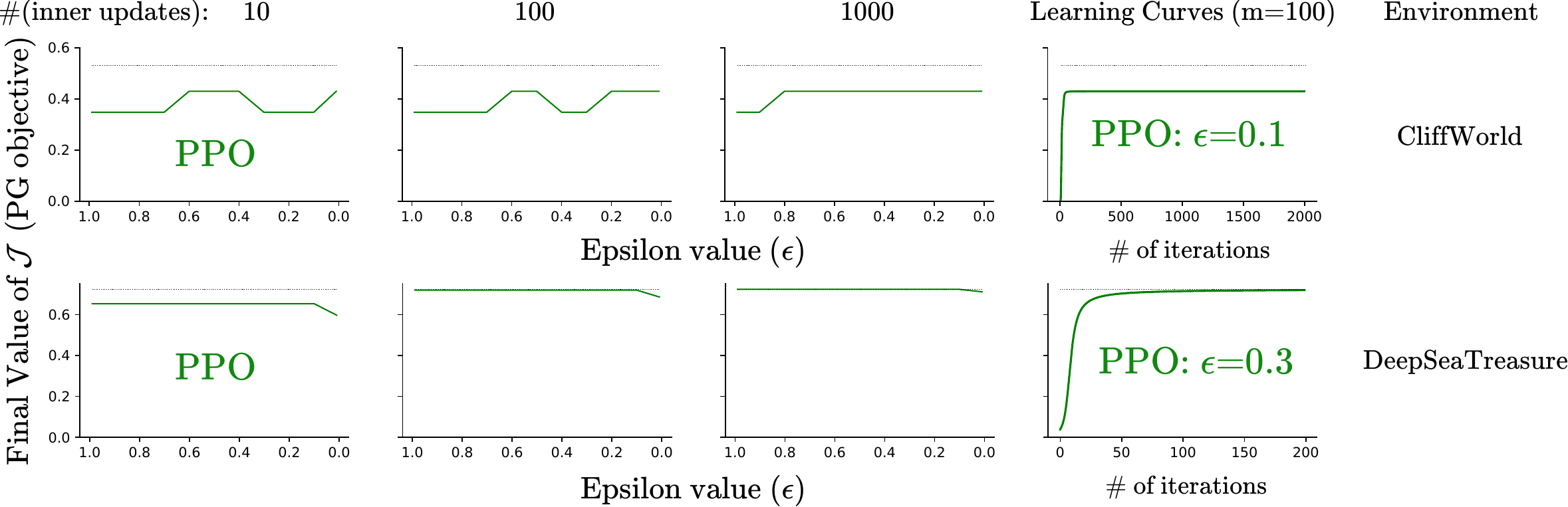}
    \caption{The parameter sensitivity plots for PPO on the CliffWorld and DeepSeaTreasure environments for different number of inner loop updates. The $x$ axis shows sweep over the clipping parameter $\epsilon$. The curve shows the final performance of the method for the best performing inner loop stepsize $\alpha$ given the $\epsilon$ value.
    \label{fig:ppo_plots}}
\end{figure*}

\subsection{Discussion} \label{app:tabular_discussion}
These experiments served to demonstrate three major points:
\begin{enumerate}
\item The optimization methods might matter as much as the policy gradient objectives being considered. We found that much of TRPO's performance came from formalizing the optimization problem as a constrained program and solving it using the optimal descent direction and a stepsize found using line search. In particular, not only did TRPO's performance suffer when we replaced the constraint with regularization, but the performance of both sMDPO and MDPO also improved significantly when we used TRPO style of optimization on their objectives. Additionally, we found that line search greatly improved the parameter sensitivity of all the algorithms.
\item The optimal $\eta$ values chosen by the \texttt{Regularized + fixed stepsize} variants of sMDPO and MDPO were much larger than the values predicted by our theoretical results. For instance, the maximal $\eta$ values for CliffWorld, as studied by the FMA-PG framework, are 
\begin{align*}
\eta_{\text{sMDPO}} &= \frac{1 - \gamma}{r_m - r_l} = \frac{1 - 0.9}{100 - (-1)} = 9.9 \times 10^{-4}, \\
\eta_{\text{MDPO}} &= \frac{(1 - \gamma)^{3}}{(r_m - r_l) \cdot 2\gamma |\mathcal{A}|} = \frac{(1 - 0.9)^{3}}{101 \times 2 \times 0.9 \times 4} = 1.4 \times 10^{-6}.
\end{align*}
Similarly for DeepSeaTreasure, they are
\begin{align*}
\eta_{\text{sMDPO}} &= \frac{1 - 0.9}{1 - (-0.01 / 5)} = 1.0 \times 10^{-2} \\
\eta_{\text{MDPO}} &= \frac{(1 - 0.9)^{3}}{(1 - (-0.01 / 5)) \times 2 \times 0.9 \times 2} = 2.8 \times 10^{-4}.
\end{align*}
Note that these values of $\eta$ are extremely small, and while the FMA-PG framework still guarantees policy improvement with these values, the convergence would be much slower than that shown in our experiments. This is natural since these bounds on $\eta$ are based on the smoothness of the policy objective $\mathcal{J}$ and from optimization literature, we know that such bounds are usually loose. Finally, also note that the optimal $\eta$ for sMDPO found by the experiments (for instance, that given in Figure \ref{fig:learning_curves}) is closer to that predicted by the theory, as compared to MDPO.
\item Each of the algorithms benefited from increasing the number of inner loop updates. These off-policy type of updates enables the PG algorithms to ``maximally squeeze'' out all the information present in the data they have already collected, thereby allowing them to improve their performance without any additional interaction with the environment. This demonstrates the strength of these methods over simpler algorithms, such as REINFORCE \citep{williams1992simple}, which have only a single update per batch of sampled data. 
\end{enumerate}

To conclude, our experiments suggest that the FMA-PG framework provides general purpose surrogate functions with policy improvement guarantees, which when combined with existing optimization techniques can yield policy gradient algorithms that are competitive to existing state-of-the-art methods.

\newpage
\section{Analytical Updates and Gradient Expressions for tabular PG Algorithms} \label{app:tabular_derivations}
In this section, we give the calculations for the closed form analytical solutions for sMDPO and MDPO, and the gradient expressions for all the four algorithms employed in our implementation for tabular PG algorithms given in Appendix \ref{app:tabular_experiments}.

\subsection{sMDPO}
We begin by considering the conventional sMDPO algorithm with a regularized objective.
\subsubsection{Closed Form Update with Softmax Representation}
Our goal is to find the closed form solution to the following optimization problem (from Eq. \ref{eq:fmd-softmax-kl-practical}, main paper):
\begin{equation}
  \pi_{t+1} = \arg\max_{\pi \in \Pi} \underbrace{\left[ \sum_s d^{\pi_t}(s) \sum_a p^{\pi_t}(a | s) \left(A^{\pi_t}(s, a) + \frac{1}{\eta} \right) \log \frac{p^\pi(s, a)}{p^{\pi_t}(s, a)} \right]}_{=: \ell^{\pi_t}_{\text{sMDPO}}}, \label{eq: optim_problem_sppo}
\end{equation}
subject to the constraints on policy $p^\pi$. We will solve this problem by assuming the policy $\pi \equiv p^\pi$ as an $|\mathcal{S}| \times |\mathcal{A}|$ table satisfying the standard constraints
\begin{align*}
  \sum_a p^\pi(a | s) &= 1,  \quad \forall s \in \mathcal{S} \\
  p^\pi(a | s) &\geq 0,  \quad \forall s \in \mathcal{S}, \; \forall a \in \mathcal{A}.
\end{align*}
We begin by formulating this problem using Lagrange multipliers $\{\lambda_s\}_{s \in \mathcal{S}}$ and $\{\lambda_{s, a}\}_{s, a \in \mathcal{S} \times \mathcal{A}}$ for all states $s$ and actions $a$:
\begin{align}
  \mathcal{L}(p^\pi, \lambda_s, \lambda_{s, a}) &= \sum_s d^{\pi_t}(s) \sum_a p^{\pi_t}(a | s) \left(A^{\pi_t}(s, a) + \frac{1}{\eta} \right) \log \frac{p^\pi(a | s)}{p^{\pi_t}(a | s)} \nonumber \\
  & \quad \; - \sum_{s, a} \lambda_{s, a} p^\pi(a | s) - \sum_s \lambda_{s} \bigg( \sum_a p^\pi(a | s) - 1 \bigg),
\end{align}
where we abused the notation, in $\mathcal{L}(p^\pi, \lambda_s, \lambda_{s, a})$, by using $\lambda_s$ to represent the set $\{\lambda_s\}_{s \in \mathcal{S}}$ and $\lambda_{s, a}$ to represent the set $\{\lambda_{s, a}\}_{s, a \in \mathcal{S} \times \mathcal{A}}$. The KKT conditions (Theorem 12.1, \citet{nocedal2006numerical}) for this constrained optimization problem can be written as:
\begin{align}
  \nabla_{p^\pi(b|x)} \mathcal{L}(p^\pi, \lambda_s, \lambda_{s, a}) &= 0, \quad \forall x \in \mathcal{S}, \; \forall b \in \mathcal{A} \tag{C1} \label{eq: KKT1} \\
  \sum_a p^\pi(a | s) &= 1, \quad \forall s \in \mathcal{S} \tag{C2} \label{eq: KKT2} \\
  p^\pi(a | s) &\geq 0, \quad \forall s \in \mathcal{S}, \; \forall a \in \mathcal{A} \tag{C3} \label{eq: KKT3} \\
  \lambda_s &\geq 0, \quad \forall s \in \mathcal{S} \tag{C4} \label{eq: KKT4} \\
  \lambda_{s} \bigg( \sum_a p^\pi(a | s) - 1 \bigg) &= 0, \quad \forall s \in \mathcal{S} \tag{C5} \label{eq: KKT5} \\
  \lambda_{s, a} p^\pi(a | s) &= 0, \quad \forall s \in \mathcal{S}, \; \forall a \in \mathcal{A}. \tag{C6} \label{eq: KKT6}
\end{align}

We now solve this system. Simplifying Eq. \ref{eq: KKT1} for an arbitrary state-action pair $(x, b)$ gives us:
\begin{align}
  \nabla_{p^\pi(b | x)} \mathcal{L}(p^\pi, \lambda_s, \lambda_{s, a}) &= d^{\pi_t}(x) p^{\pi_t}(b|x) \left( A^{\pi_t}(x, b) + \frac{1}{\eta} \right) \frac{1}{p^\pi(b|x)} - \lambda_{x, b} - \lambda_x = 0 \nonumber \\
  \Rightarrow \qquad \qquad \qquad \quad p^\pi(b | x) &= \frac{d^{\pi_t}(x) p^{\pi_t}(b|x) (1 + \eta A^{\pi_t}(x, b))}{\eta (\lambda_x + \lambda_{x, b})}. \label{eq: lagrangian_derivative_sppo}
\end{align}
Let us set 
\begin{equation}
  \lambda_{s, a} = 0, \quad \forall s \in \mathcal{S}, \; \forall a \in \mathcal{A}.
\end{equation}
Combining Eq. \ref{eq: lagrangian_derivative_sppo} with the second KKT condition gives us
\begin{equation}
  \lambda_s = \frac{1}{\eta} \sum_a d^{\pi_t}(s) p^{\pi_t}(a|s) (1 + \eta A^{\pi_t}(s, a)).
\end{equation}
Therefore, with the standard coverage assumption $d^{\pi_t}(s) > 0$, $p^\pi(a | s)$ becomes
\begin{equation}
  p^\pi(a | s) = \frac{p^{\pi_t}(a|s) (1 + \eta A^{\pi_t}(s, a))}{\sum_b p^{\pi_t}(b|s) (1 + \eta A^{\pi_t}(s, b))}.
\end{equation}
Note that $d^{\pi_t}(s), p^{\pi_t}(a|s) \geq 0$ for any state-action pair, since they are proper measures. We also need to ensure that
\begin{equation*}
  1 + \eta A^{\pi_t}(s, a) \geq 0
\end{equation*}
to satisfy the third and fourth KKT conditions. One straightforward way to achieve this is to define $p^\pi(a | s) = 0$ whenever $1 + \eta A^{\pi_t}(s, a) < 0$, and accordingly re-define $\lambda_s$. This gives us the final solution to our original optimization problem (Eq. \ref{eq: optim_problem_sppo}):
\begin{equation}
  \pi_{t+1} = p^\pi(s, a) = \frac{p^{\pi_t}(a|s) \max(1 + \eta A^{\pi_t}(s, a), 0)}{\sum_b p^{\pi_t}(b|s) \max(1 + \eta A^{\pi_t}(s, b), 0)}.
\end{equation}
However, it leaves us one last problem to deal with: ensuring that for any state $s$, there always exists at least one action $a$, such that $1 + \eta A^{\pi_t}(s, a) > 0$. This is not a problem since we can put a condition on $\eta$ in order to fulfill this constraint.
% Is it always true that given ? Because otherwise, we would fail to satisfy the second KKT condition. 

\subsubsection{Gradient of the Loss Function with Tabular Softmax Policy Parameterization}
Consider the softmax policy parameterization
\begin{equation}
  p^\pi(b | x) = \frac{e^{\theta(x, b)}}{\sum_c e^{\theta(x, c)}}, \label{eq: softmax}
\end{equation}
where $\theta(x, b)$ for all state-action pairs $(x, b)$ are action preferences maintained in a table (tabular parameterization). Also note that the derivative of the policy with respect to the action preferences is given by
\begin{equation}
  \frac{\partial}{\partial \theta(s, a)} p^\pi(b | x) = \mathbb{I}(x = s) \Big( \mathbb{I}(b = a) - p^\pi(a | x) \Big) p^\pi(b | x),
\end{equation}
where $\mathbb{I}(a = b)$ is the identity function when $a = b$ and zero otherwise. 
We will use gradient ascent to approximately solve Eq. \ref{eq: optim_problem_sppo}; to do that, the quantity of interest is
\begin{align}
  \frac{\partial}{\partial \theta(s, a)} \ell^{\pi_t}_{\text{sMDPO}} &= \sum_{x \in \mathcal{S}} \sum_{b \in \mathcal{A}} \left[ \frac{\partial}{\partial \theta(s, a)} p^\pi(b | x) \right] \left[ \frac{\partial}{\partial p^\pi(b | x)} \ell^{\pi_t}_{\text{sMDPO}} \right] \tag*{(using total derivative)} \\
  &= \sum_{x, b} \Big[ \mathbb{I}(x = s) \Big( \mathbb{I}(b = a) - p^\pi(a | x) \Big) p^\pi(b | x) \Big] \left[ d^{\pi_t}(x) p^{\pi_t}(b|x) \left( A^{\pi_t}(x, b) + \frac{1}{\eta} \right) \frac{1}{p^\pi(b|x)} \right] \nonumber \\
  &= \E_{X \sim d^{\pi_t}, B \sim p^{\pi_t}(\cdot | X)} \left[ \mathbb{I}(X = s) \Big( \mathbb{I}(B = a) - p^\pi(a | x) \Big) \left( A^{\pi_t}(X, B) + \frac{1}{\eta} \right) \right] \\
  &= d^{\pi_t}(s) \sum_b \Big( \mathbb{I}(b = a) - p^\pi(a | s) \Big) p^{\pi_t}(b|s) \left( A^{\pi_t}(s, b) + \frac{1}{\eta} \right) \nonumber \\
  &= d^{\pi_t}(s) \left[ p^{\pi_t}(a|s) \left( A^{\pi_t}(s, a) + \frac{1}{\eta} \right) - p^\pi(a | s) \sum_b p^{\pi_t}(b|s) \left(A^{\pi_t}(s, b) + \frac{1}{\eta} \right) \right] \nonumber \\
  &= d^{\pi_t}(s) \left[ p^{\pi_t}(a|s) \left( A^{\pi_t}(s, a) + \frac{1}{\eta} \right) - \frac{p^\pi(a | s)}{\eta} \right], \nonumber
\end{align}
Now we can simply update the inner loop of FMA-PG (Algorithm 1, main paper) via gradient ascent:
\begin{equation}
  \theta(s, a) \; \leftarrow \; \theta(s, a) + \alpha d^{\pi_t}(s) \left[ p^{\pi_t}(a|s) \left( A^{\pi_t}(s, a) + \frac{1}{\eta} \right) - \frac{p^\pi(a | s)}{\eta} \right].
\end{equation}

\subsection{Mirror Descent Policy Optimization (MDPO)}
In this section, we study the MDPO type FMA-PG update (Eq. \ref{eq:fmd-softmax-statewise} in main paper). We first calculate the analytical solution to that optimization problem, and then calculate its gradient which we use in the experiments. However, in the analysis that follows, we we replace the advantage function $A^{\pi_t}$ with the action-value function $Q^{\pi_t}$ to make it exactly same as the original MDPO \citep{tomar2020mirror} update.

\subsubsection{Closed Form Update with Direct Representation}
While giving the MDPO type FMA-PG equation (Eq. \ref{eq:fmd-softmax-statewise}), the paper considers the direct representation along with tabular parameterization of the policy, albeit with a small change in notation as compared to the previous subsection: $\pi(a|s) \equiv p^\pi(a|s, \theta)$. However, since this notation is more cumbersome, we will stick with our the notation of the previous subsection: $\pi(a|s) \equiv p^\pi(a|s)$. The constraints on the parameters $p^\pi(s, a)$ are the same as before: $\sum_a p^\pi(a | s) = 1, \; \forall s \in \mathcal{S}$; and $p^\pi(a | s) \geq 0, \; \forall s \in \mathcal{S}, \; \forall a \in \mathcal{A}$. Our goal, this time, is to solve the following optimization problem (from Eq. 6, main paper)
\begin{equation}
  \pi_{t+1} = \arg\max_{\pi \in \Pi} \underbrace{\left[ \sum_s d^{\pi_t}(s) \sum_a p^{\pi_t}(a|s) \left( Q^{\pi_t}(s, a) \frac{p^\pi(a | s)}{p^{\pi_t}(a | s)} - \frac{1}{\eta} D_\phi (p^\pi(\cdot | s), p^{\pi_t}(\cdot | s)) \right) \right]}_{=: \ell^{\pi_t}_{\text{MDPO}}}, \label{eq: optim_problem_mdpo}
\end{equation}
with the mirror map as the negative entropy (Eq. 5.27, \citet{beck2003mirror}). This particular choice of the mirror map simplifies the Bregman divergence as follows
\begin{equation}
  D_\phi (p^\pi(\cdot | s), p^{\pi_t}(\cdot | s)) = \text{KL}(p^\pi(\cdot | s) \| p^{\pi_t}(\cdot | s)) := \sum_a p^\pi(a | s) \log \frac{p^\pi(a | s)}{p^{\pi_t}(a | s)}.
\end{equation}
The optimization problem (Eq. \ref{eq: optim_problem_mdpo}) then simplifies to
\begin{equation}
  \pi_{t+1} = \arg\max_{\pi \in \Pi} \left[ \sum_s d^{\pi_t}(s) \sum_a p^{\pi_t}(a|s) \left( Q^{\pi_t}(s, a) \frac{p^\pi(a | s)}{p^{\pi_t}(a | s)} - \frac{1}{\eta} \sum_{a'} p^\pi(a' | s) \log \frac{p^\pi(a' | s)}{p^{\pi_t}(a' | s)} \right) \right].
\end{equation}

Proceeding analogously to the previous subsection, we use Lagrange multipliers $\lambda_s$, $\lambda_{s, a}$ for all states $s$ and actions $a$ to obtain the function
\begin{align}
  \mathcal{L}(p^\pi, \lambda_s, \lambda_{s, a}) &= \sum_s d^{\pi_t}(s) \sum_a p^{\pi_t}(a|s) Q^{\pi_t}(s, a) \frac{p^\pi(a | s)}{p^{\pi_t}(a | s)} - \frac{1}{\eta} \sum_s d^{\pi_t}(s) \sum_{a'} p^\pi(a' | s) \log \frac{p^\pi(a' | s)}{p^{\pi_t}(a' | s)} \nonumber \\
  & \quad \; - \sum_{s, a} \lambda_{s, a} p^\pi(a | s) - \sum_s \lambda_{s} \bigg( \sum_a p^\pi(a | s) - 1 \bigg).
\end{align}
The KKT conditions are exactly the same as before (Eq. \ref{eq: KKT1} to Eq. \ref{eq: KKT6}).

Again, we begin by solving the first KKT condition:
\begin{align}
  \nabla_{p^\pi(b | x)} \mathcal{L}(p^\pi, \lambda_s, \lambda_{s, a}) &= d^{\pi_t}(x) p^{\pi_t}(b|x) \frac{Q^{\pi_t}(x, b)}{p^{\pi_t}(b | x)} - \frac{d^{\pi_t}(x)}{\eta} \left[ \log \frac{p^\pi(b | x)}{p^{\pi_t}(b | x)} + 1 \right] - \lambda_{x, b} - \lambda_x \nonumber \\
  &= \frac{d^{\pi_t}(x)}{\eta} \left[ \eta Q^{\pi_t}(x, b) - \log \frac{p^\pi(b | x)}{p^{\pi_t}(b | x)} - 1 - \frac{\eta (\lambda_{x, b} + \lambda_x)}{d^{\pi_t}(x)} \right] \nonumber \\
  &= 0 \nonumber \\
  \Rightarrow \qquad \qquad \log \frac{p^\pi(b | x)}{p^{\pi_t}(b | x)} &= \eta Q^{\pi_t}(x, b) - \frac{\eta (\lambda_{x, b} + \lambda_x)}{d^{\pi_t}(x)} - 1 \nonumber \\
  \Rightarrow \qquad \qquad \qquad p^\pi(b | x) &= p^{\pi_t}(b | x) \cdot e^{\eta Q^{\pi_t}(x, b)} \cdot e^{- \frac{\eta (\lambda_{x, b} + \lambda_x)}{d^{\pi_t}(x)} - 1}, \label{eq: lagrangian_derivative_mdpo}
\end{align}
where in the fourth line, we used the assumption that $d^{\pi_t}(x) > 0$ for all states $x$. We again set
\begin{equation}
  \lambda_{s, a} = 0, \quad \forall s \in \mathcal{S}, \; \forall a \in \mathcal{A}.
\end{equation}
And, we put Eq. \ref{eq: lagrangian_derivative_mdpo} in the second KKT condition to get
\begin{equation}
  e^{- \frac{\eta \lambda_x}{d^{\pi_t}(x)} - 1} = \left( \sum_b p^{\pi_t}(b | x) \cdot e^{\eta Q^{\pi_t}(x, b)} \right)^{-1}.
\end{equation}
Therefore, we obtain
\begin{equation}
  p^\pi(a | s) = \frac{p^{\pi_t}(a | s) \cdot e^{\eta Q^{\pi_t}(s, a)}}{\sum_b p^{\pi_t}(b | s) \cdot e^{\eta Q^{\pi_t}(s, b)}}.
\end{equation}
This leaves us one last problem to deal with: ensuring $\lambda_s \geq 0$ for all states $s$. Again, we can set the step-size $\eta$ to ensure this constraint. 

\subsubsection{Gradient of the MDPO Loss Function with Tabular Softmax Parameterization}
We again use the tabular softmax policy parameterization given by Eq. \ref{eq: softmax}, and compute $\nabla_{\theta(s, a)} \ell^{\pi_t}_{\text{MDPO}}$ for the MDPO loss (we substitute $Q^{\pi_t}$ with $A^{\pi_t}$ in this calculation):
\begin{align}
  \frac{\partial}{\partial \theta(s, a)} \ell^{\pi_t}_{\text{MDPO}} &= \sum_{x, b} \left[ \frac{\partial}{\partial \theta(s, a)} p^\pi(b | x) \right] \left[ \frac{\partial}{\partial p^\pi(b | x)} \ell^{\pi_t}_{\text{MDPO}} \right] \tag*{(using total derivative)} \\
  &= \sum_{x, b} \Big[ \mathbb{I}(x = s) \Big( \mathbb{I}(b = a) - p^\pi(a | x) \Big) p^\pi(b | x) \Big] \left[ \frac{d^{\pi_t}(x)}{\eta} \left( \eta A^{\pi_t}(x, b) - \log \frac{p^\pi(b | x)}{p^{\pi_t}(b | x)} - 1 \right) \right] \nonumber \\
  &= \frac{d^{\pi_t}(s)}{\eta} \sum_b \Big( \mathbb{I}(b = a) - p^\pi(a | s) \Big) p^\pi(b | s) \left[ \eta A^{\pi_t}(s, b) - \log \frac{p^\pi(b | s)}{p^{\pi_t}(b | s)} - 1 \right] \nonumber \\
  &= \frac{d^{\pi_t}(s)}{\eta} p^\pi(a | s) \left[ \eta A^{\pi_t}(s, a) - \eta \sum_b p^\pi(b|s) A^{\pi_t}(s, b) - \log \frac{p^\pi(a | s)}{p^{\pi_t}(a | s)} + \text{KL}(p^\pi(\cdot | s) \| p^{\pi_t}(\cdot | s)) \right], \nonumber
\end{align}
where in the last line, we used the fact that
\begin{equation*}
  \sum_b p^\pi(b | s) \left[ \eta A^{\pi_t}(s, b) - \log \frac{p^\pi(b | s)}{p^{\pi_t}(b | s)} - 1 \right] = \eta \sum_b p^\pi(b|s) A^{\pi_t}(s, b) - \text{KL}(p^\pi(\cdot | s) \| p^{\pi_t}(\cdot | s)) - 1.
\end{equation*}

\subsection{Trust Region Policy Optimization (TRPO)} \label{app:trpo}
At each step of the policy update, TRPO (Eq. 14, \citet{schulman2015trust}) solves the following problem:
\begin{equation}
  \max_\theta \; \underbrace{\sum_s d^{\pi_t}(s) \sum_a p^{\pi_\theta}(a | s) Q^{\pi_t}(s, a)}_{=: \mathcal{J}_{\text{TRPO}}} \qquad \text{subject to } \underbrace{\sum_s d^{\pi_t}(s) \cdot \text{KL}(p^{\pi_t}(\cdot | s) \| p^{\pi_\theta}(\cdot | s))}_{=: \mathcal{C}_{\text{TRPO}}} \leq \delta.  
\end{equation}
Unlike the sMDPO and the MDPO updates, an analytical solution cannot be derived for this update (since it would require solving a system of non-trivial non-linear equations). Therefore, we will use gradient based methods to approximately solve this problem. From Appendix C of \citet{schulman2015trust}, the descent direction is given by $s \approx A^{-1} g$ where the vector $g$ is defined as $g_{(s, a)} := \frac{\partial}{\partial \theta(s, a)} \mathcal{J}_{\text{TRPO}}$, and the matrix $A$ is defined as $A_{(s, a), (s', a')} := \frac{\partial}{\partial \theta(s, a)} \frac{\partial}{\partial \theta(s', a')} \mathcal{C}_{\text{TRPO}}$. We analytically compute the expression for this direction assuming a softmax policy (Eq. \ref{eq: softmax}). The vector $g$ can be readily calculated as
\begin{align}
  \frac{\partial}{\partial \theta(s, a)} \mathcal{J}_{\text{TRPO}} &= \sum_x d^{\pi_t}(x) \sum_b Q^{\pi_t}(x, b) \frac{\partial p^{\pi_\theta}(b | x)}{\partial \theta(s, a)} \nonumber \\
  &= \sum_x d^{\pi_t}(x) \sum_b Q^{\pi_t}(x, b) \mathbb{I}(x = s) \Big( \mathbb{I}(b = a) - p^{\pi_\theta}(a | x) \Big) p^{\pi_\theta}(b | x) \nonumber \\
  &= \sum_x d^{\pi_t}(x) \mathbb{I}(x = s) \left[ \sum_b \mathbb{I}(b = a) p^{\pi_\theta}(b | x) Q^{\pi_t}(x, b) - p^{\pi_\theta}(a | x) \sum_b p^{\pi_\theta}(b | x) Q^{\pi_t}(x, b) \right] \nonumber \\
  &= d^{\pi_t}(s) p^{\pi_\theta}(a | s) \left[ Q^{\pi_t}(s, a) - \sum_b p^{\pi_\theta}(b | s) Q^{\pi_t}(s, b) \right]. \label{eq: trpo_gradient}
\end{align}
For calculating the matrix $A$, we use the law of total derivative to obtain
\begin{align}
  \frac{\partial}{\partial \theta(s, a)} \mathcal{C}_{\text{TRPO}} &= \sum_{x, b} \left[ \frac{\partial}{\partial \theta(s, a)} p^{\pi_\theta}(b | x) \right] \left[ \frac{\partial}{\partial p^{\pi_\theta}(b | x)} \sum_s d^{\pi_t}(s) \sum_a p^{\pi_t}(a | s) \log \frac{p^{\pi_t}(a | s)}{p^{\pi_\theta}(a | s)} \right] \nonumber \\
  &= \sum_{x, b} \left[ \mathbb{I}(x = s) \Big( \mathbb{I}(b = a) - p^{\pi_\theta}(a | x) \Big) p^{\pi_\theta}(b | x) \right] \left[ - d^{\pi_t}(x) \frac{p^{\pi_t}(b | x)}{p^{\pi_\theta}(b | x)} \right] \nonumber \\
  &= - d^{\pi_t}(s) \sum_b \Big( \mathbb{I}(b = a) - p^{\pi_\theta}(a | s) \Big) p^{\pi_t}(b | s) \nonumber \\
  &= - d^{\pi_t}(s) \left[ \sum_b \mathbb{I}(b = a) p^{\pi_t}(b | s) - p^{\pi_\theta}(a | s) \sum_b p^{\pi_t}(b | s) \right] \nonumber \\
  &= d^{\pi_t}(s) \Big[ p^{\pi_\theta}(a | s) - p^{\pi_t}(a | s) \Big]. \label{eq: trpo_kl_mid}
\end{align}
Finally, using the above result yields
\begin{align}
  \frac{\partial}{\partial \theta(s, a)} \frac{\partial}{\partial \theta(s', a')} \mathcal{C}_{\text{TRPO}} &= \frac{\partial}{\partial \theta(s, a)} d^{\pi_t}(s') \Big[ p^{\pi_\theta}(a' | s') - p^{\pi_t}(a' | s') \Big] \nonumber \\
  &= d^{\pi_t}(s') \cdot \frac{\partial}{\partial \theta(s, a)} p^{\pi_\theta}(a' | s') \nonumber \\
  &= \mathbb{I}(s'=s) \cdot d^{\pi_t}(s') \Big( \mathbb{I}(a'=a) - p^{\pi_\theta}(a | s') \Big) p^{\pi_\theta}(a' | s') \\
  \Rightarrow \qquad \qquad \qquad \quad A_{(s, :), (s, :)} &= d^{\pi_t}(s) \Big( \text{diag} (p^{\pi_\theta}(\cdot | s)) - p^{\pi_\theta}(\cdot | s) p^{\pi_\theta}(\cdot | s)^\top \Big),
\end{align}
where $p^{\pi_\theta}(\cdot | s) \in \mathbb{R}^{|\mathcal{A}|}$ is the vector defined as $[p^{\pi_\theta}(\cdot | s)]_a = p^{\pi_\theta}(a | s)$ and $A_{(s, :), (s, :)}$ denotes the square sub-block of the matrix $A$ corresponding to the given state $s$ and all the actions. In our experiments, since our $A$ matrix is small, we directly take its inverse to compute the update direction, thereby bypassing the conjugate method. Once we have the update direction, we then compute the maximal stepsize $\beta$ and perform a backtracking line search similar to the TRPO paper.

\subsection{Proximal Policy Optimization (PPO)} \label{app:ppo}
The Proximal Policy Optimization algorithm \citep{schulman2017proximal} solves the following optimization problem at each iteration step:
\begin{equation}
  \max_\theta \; \underbrace{\sum_s d^{\pi_t}(s) \sum_a p^{\pi_t}(a | s) \cdot \min \left( \begin{matrix} \frac{p^{\pi_\theta}(a | s)}{p^{\pi_t}(a | s)} A^{\pi_t}(s, a), \\ \text{clip} \left[\frac{p^{\pi_\theta}(a | s)}{p^{\pi_t}(a | s)}, 1 - \epsilon, 1 + \epsilon \right] A^{\pi_t}(s, a) \end{matrix} \right)}_{=: \mathcal{J}_{\text{PPO}}}.
\end{equation}
The gradient of the objective $\mathcal{J}_{\text{PPO}}$ can be shown to be equivalent to
\begin{equation}
  \nabla \mathcal{J}_{\text{PPO}} = \sum_s d^{\pi_t}(s) \sum_a p^{\pi_t}(a | s) \cdot \mathbb{I} \Big( \text{cond}(s, a) \Big) \frac{\nabla p^{\pi_\theta}(a | s)}{p^{\pi_t}(a | s)} A^{\pi_t}(s, a),
\end{equation}
where 
\begin{equation}
  \text{cond}(s, a) = \left( A^{\pi_t}(s, a) > 0 \;\bigwedge\; \frac{p^{\pi_\theta}(a | s)}{p^{\pi_t}(a | s)} < 1 + \epsilon \right) \;\bigvee\; \left( A^{\pi_t}(s, a) < 0 \;\bigwedge\; \frac{p^{\pi_\theta}(a | s)}{p^{\pi_t}(a | s)} > 1 - \epsilon \right).
\end{equation}
Repeating our usual drill, we assume a softmax policy to obtain:
\begin{align}
  & \frac{\partial}{\partial \theta(s, a)} \mathcal{J}_{\text{PPO}} \nonumber \\
  &= \sum_x d^{\pi_t}(x) \sum_b \mathbb{I} \Big( \text{cond}(x, b) \Big) \frac{\partial p^{\pi_\theta}(b | x)}{\partial \theta(s, a)} A^{\pi_t}(x, b) \nonumber \\
  &= \sum_x d^{\pi_t}(x) \sum_b \mathbb{I} \Big( \text{cond}(x, b) \Big) \mathbb{I}(x = s) \Big( \mathbb{I}(b = a) - p^{\pi_\theta}(a | x) \Big) p^{\pi_\theta}(b | x) A^{\pi_t}(x, b) \nonumber \\
  &= d^{\pi_t}(s) \Bigg[ \sum_b \mathbb{I}(b = a) \mathbb{I} \Big( \text{cond}(s, b) \Big) p^{\pi_\theta}(b | s) A^{\pi_t}(s, b) - p^{\pi_\theta}(a | s) \sum_b \mathbb{I} \Big( \text{cond}(s, b) \Big) p^{\pi_\theta}(b | s) A^{\pi_t}(s, b) \Bigg] \nonumber \\
    &= d^{\pi_t}(s) p^{\pi_\theta}(a | s) \left[ \mathbb{I} \Big( \text{cond}(s, a) \Big) A^{\pi_t}(s, a) - \sum_b p^{\pi_\theta}(b | s) \mathbb{I} \Big( \text{cond}(s, b) \Big) A^{\pi_t}(s, b) \right]. \label{eq: ppo_gradient}
\end{align}
The PPO gradient (Eq. \ref{eq: ppo_gradient}) is exactly the same as the TRPO gradient (Eq. \ref{eq: trpo_gradient}) except for the additional condition on choosing only specific state-action pairs while calculating the difference between advantage under the current policy and the approximate change in advantage under the updated policy.

\subsection{MDPO with Constraints}
In this section, we calculate the second derivative of the MDPO constraint as given in Table \ref{table:ablation_study}. This will allow us compute the Hessian $A^{\text{MDPO}}$, which is the analog of the $A$ matrix from TRPO implementation, and help us implement MDPO with a constrained objective and line search.

Continuing from the gradient of the MDPO constraint given in Table \ref{table:ablation_study_grad}, we get
\begin{align}
  \frac{\partial}{\partial p^\pi(b | x)} \frac{\partial \mathcal{C}_{\text{MDPO}}}{\partial \theta(s', a')} &= \frac{\partial}{\partial p^\pi(b | x)} d^{\pi_t}(s') p^\pi(a' | s') \left( \log \frac{p^\pi(a' | s')}{p^{\pi_t}(a' | s')} - \sum_c p^\pi(c | s') \log \frac{p^\pi(c | s')}{p^{\pi_t}(c | s')} \right) \nonumber \\
  &= \mathbb{I}(x = s') d^{\pi_t}(s') \Bigg[ \mathbb{I}(b = a') \left( \log \frac{p^\pi(a' | s')}{p^{\pi_t}(a' | s')} - \sum_c p^\pi(c | s') \log \frac{p^\pi(c | s')}{p^{\pi_t}(c | s')} \right) \nonumber \\
    & \qquad \qquad \qquad \qquad \quad + p^\pi(a' | s') \frac{\mathbb{I}(b = a')}{p^\pi(a' | s')} - p^\pi(a' | s') \left( \log \frac{p^\pi(b | s')}{p^{\pi_t}(b | s')} + 1 \right) \Bigg] \nonumber \\
  &= \mathbb{I}(x = s') d^{\pi_t}(s') \Bigg[ \mathbb{I}(b = a') \cdot T_{\text{itd}}(s', a') - p^\pi(a' | s') \left( \log \frac{p^\pi(b | s')}{p^{\pi_t}(b | s')} + 1 \right) \Bigg],
\end{align}
where we introduced an intermediate variable $T_{\text{itd}}(s', a') := \log \frac{p^\pi(a' | s')}{p^{\pi_t}(a' | s')} - \text{KL} ( p^\pi(\cdot | s') \| p^{\pi_t}(\cdot | s') ) + 1$. Now, using the law of total derivative, we obtain
\begin{align}
  \frac{\partial}{\partial \theta(s, a)} \frac{\partial \mathcal{C}_{\text{MDPO}}}{\partial \theta(s', a')} &= \sum_{x, b} \frac{\partial p^\pi(b | x)}{\partial \theta(s, a)} \times \frac{\partial}{\partial p^\pi(b | x)} \frac{\partial \mathcal{C}_{\text{MDPO}}}{\partial \theta(s', a')} \nonumber \\
  &= \sum_{x, b} \mathbb{I}(x = s) \Big[ \mathbb{I}(b = a) - p^\pi(a | x) \Big] p^\pi(b | x) \times \mathbb{I}(x = s') d^{\pi_t}(s') \nonumber \\
  & \qquad \quad \times \Bigg[ \mathbb{I}(b = a') \cdot T_{\text{itd}}(s', a') - p^\pi(a' | s') \left( \log \frac{p^\pi(b | s')}{p^{\pi_t}(b | s')} + 1 \right) \Bigg] \nonumber \\
  &= \mathbb{I}(s = s') d^{\pi_t}(s) \cdot T_{\text{aux}},
\end{align}
where the auxillary term $T_{\text{aux}}$ is 
\begin{align}
  T_{\text{aux}} &:= \sum_b \Big[ \mathbb{I}(b = a) - p^\pi(a | s) \Big] p^\pi(b | s) \Bigg[ \mathbb{I}(b = a') \cdot T_{\text{itd}}(s, a') - p^\pi(a' | s) \left( \log \frac{p^\pi(b | s)}{p^{\pi_t}(b | s)} + 1 \right) \Bigg] \\
  &= T_{\text{itd}}(s, a') \sum_{b} \mathbb{I}(b = a) p^\pi(b | s) \mathbb{I}(b = a') - p^\pi(a' | s) \sum_{b} \mathbb{I}(b = a) p^\pi(b | s) \left( \log \frac{p^\pi(b | s)}{p^{\pi_t}(b | s)} + 1 \right) \nonumber \\
  & \qquad - p^\pi(a | s) T_{\text{itd}}(s, a') \sum_{b} p^\pi(b | s) \mathbb{I}(b = a') + p^\pi(a' | s) p^\pi(a | s) \sum_{b} p^\pi(b | s) \left( \log \frac{p^\pi(b | s)}{p^{\pi_t}(b | s)} + 1 \right) \nonumber \\  
  &= T_{\text{itd}}(s, a') p^\pi(a | s) \mathbb{I}(a = a') - p^\pi(a' | s) p^\pi(a | s) \left( \log \frac{p^\pi(a | s)}{p^{\pi_t}(a | s)} + 1 \right) \nonumber \\
  & \qquad - p^\pi(a | s) T_{\text{itd}}(s, a') p^\pi(a' | s) + p^\pi(a' | s) p^\pi(a | s) \Big( \text{KL}( p^\pi(\cdot | s) \| p^{\pi_t}(\cdot | s) ) + 1 \Big) \nonumber \\
  &= \mathbb{I}(a = a') p^\pi(a | s) T_{\text{itd}}(s, a') - p^\pi(a' | s) p^\pi(a | s) \Big[ T_{\text{itd}}(s, a') + T_{\text{itd}}(s, a) \Big] + p^\pi(a' | s) p^\pi(a | s).
\end{align}
Therefore, 
\begin{align}
  \frac{\partial}{\partial \theta(s, a)} \frac{\partial \mathcal{C}_{\text{MDPO}}}{\partial \theta(s', a')} &= \mathbb{I}(s = s') d^{\pi_t}(s) \bigg[ \mathbb{I}(a = a') p^\pi(a | s) T_{\text{itd}}(s, a') - p^\pi(a | s) p^\pi(a' | s) T_{\text{itd}}(s, a') \nonumber \\
    & \qquad \qquad \qquad \qquad \quad - p^\pi(a' | s) p^\pi(a | s) T_{\text{itd}}(s, a) + p^\pi(a' | s) p^\pi(a | s) \bigg] \\
  \Rightarrow \qquad \qquad A^{\text{MDPO}}_{(s, :), (s, :)} &= d^{\pi_t}(s) \cdot \Big[ \text{diag} \big( T_{\text{vec}}(s) \big) - p^{\pi}(\cdot | s) T_{\text{vec}}(s)^\top \nonumber \\
    & \qquad \qquad \qquad - T_{\text{vec}}(s) p^{\pi}(\cdot | s)^\top + p^{\pi}(\cdot | s) p^{\pi}(\cdot | s)^\top \Big],
\end{align}
where we introduced yet another intermediate term $T_{\text{vec}}(s)$, defined as
\begin{align}
  T_{\text{vec}}(s) &:= p^{\pi}(\cdot | s) \odot T_{\text{itd}}(s, \cdot) \\
  &= p^{\pi}(\cdot | s) \odot \left[ \log \Big( p^\pi(\cdot | s) \oslash p^{\pi_t}(\cdot | s) \Big) - \text{KL} (p^\pi(\cdot | s) \| p^{\pi_t}(\cdot | s)) \mathbf{1}_{|\mathcal{A}|} + \mathbf{1}_{|\mathcal{A}|} \right],
\end{align}
and $\oslash$ in the above equation represents the elementwise vector division defined as $[a \oslash b]_i := a_i / b_i$ for any two vectors $a$ and $b$. As a sanity check, note that the matrix $A^{\text{MDPO}}$ is symmetric, as any Hessian matrix should be.
\newpage
\section{Additional experiments on MuJoCo environments}
\label{app:mujoco}
In this section, we present results on a series of MuJoCo environments where learning rate decay and gradient clipping have \emph{not} been applied.
\cref{fig:sppo_mujoco} shows that, while sPPO (in orange) still learns something, PPO is unable to make progress, regardless of the capping ($\epsilon$) and the number of inner loop steps ($m$), further reinforcing our intuition that the softmax paramaterization leads to a more robust optimization.

\begin{figure}[h]
    \centering
    \includegraphics[width=\textwidth]{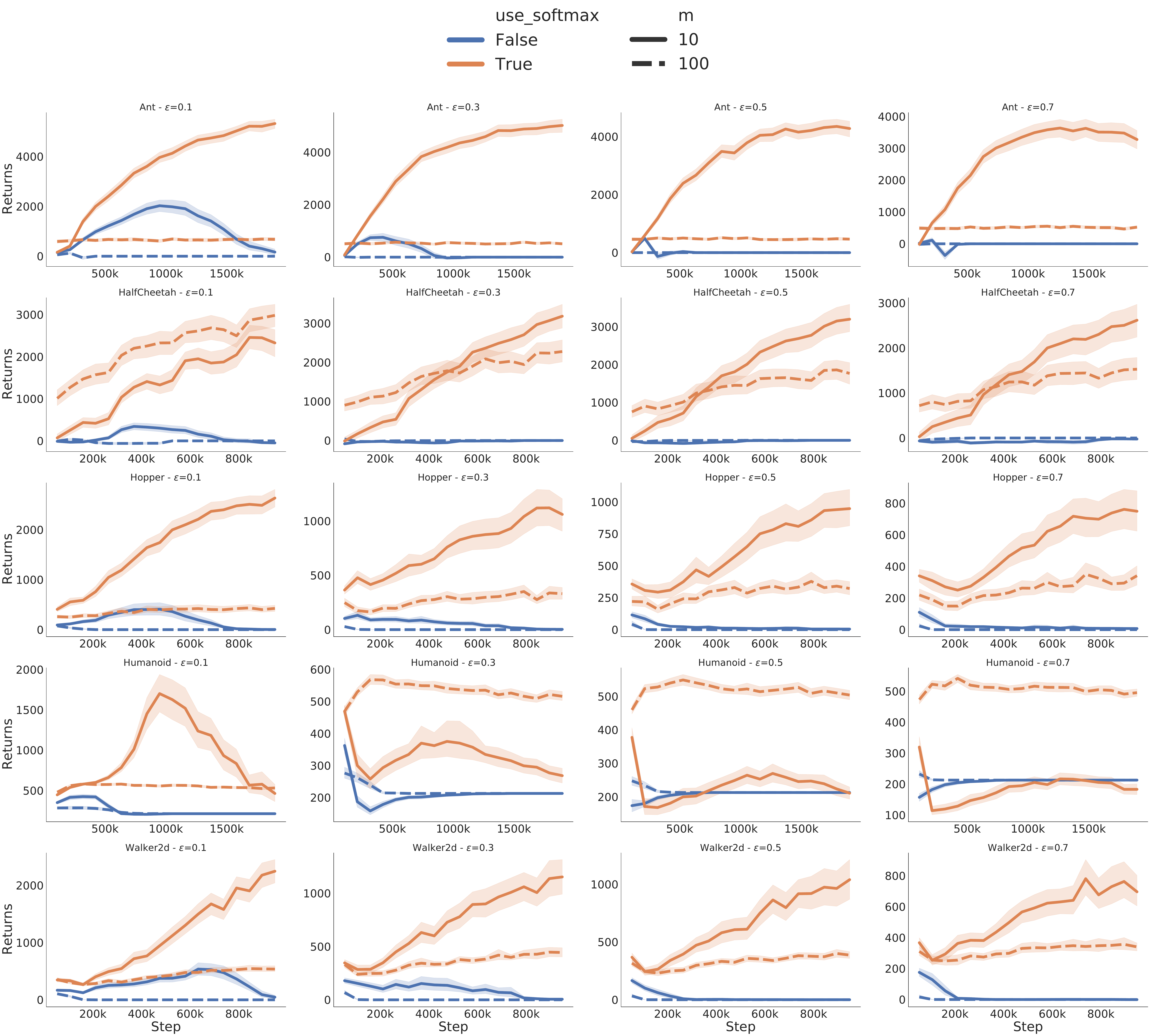}
    \caption{Average discounted return and 95\% confidence interval (over 180 runs) for {\color{blue}PPO} and {\color{orange}softmax PPO} on 4 environments (\emph{env} - rows) and for four different clipping strengths (\emph{epsilon} - columns). We see that sPPO is more robust to large values of clipping, even more so when the number of updates in the inner loop grows (linestyle).
    \label{fig:sppo_mujoco}}
\end{figure}

\end{document}